\documentclass[preprint]{article}
\usepackage[page,toc,titletoc,title]{appendix}

\usepackage{microtype}
\usepackage{graphicx}
\usepackage{subfigure}
\usepackage{booktabs} 
\usepackage{xargs}  
\usepackage{hyperref}
\usepackage[super]{nth}
\usepackage{subcaption}
\usepackage{graphicx}

\usepackage{titlesec}
\usepackage{icml2024}

\usepackage{amsmath,amsfonts,bm}

\usepackage{letltxmacro}
\LetLtxMacro{\originaleqref}{\eqref}

\def\eqref#1{equation~\originaleqref{#1}}

\def\1{\bm{1}}

\def\vb{{\bm{b}}}

\def\vd{{\bm{d}}}
\def\ve{{\bm{e}}}
\def\vf{{\bm{f}}}

\def\vm{{\bm{m}}}
\def\vn{{\bm{n}}}

\def\vr{{\bm{r}}}

\def\vu{{\bm{u}}}
\def\vv{{\bm{v}}}

\def\vx{{\bm{x}}}

\def\vz{{\bm{z}}}

\def\mR{{\bm{R}}}

\DeclareMathAlphabet{\mathsfit}{\encodingdefault}{\sfdefault}{m}{sl}
\SetMathAlphabet{\mathsfit}{bold}{\encodingdefault}{\sfdefault}{bx}{n}

\usepackage{amsmath}
\usepackage{amssymb}
\usepackage{mathtools}
\usepackage{amsthm}
\usepackage[capitalize,noabbrev]{cleveref}
\theoremstyle{plain}
\newtheorem{theorem}{Theorem}[section]
\newtheorem{proposition}[theorem]{Proposition}

\theoremstyle{definition}

\theoremstyle{remark}

\DeclareMathOperator{\atantwo}{atan2}
\newcommand{\eqq}[1]{(\ref{#1})}
\newcommand{\mlp}[1]{\emph{SO2-MLP}}

\usepackage{comment}

\usepackage[colorinlistoftodos,prependcaption,textsize=tiny]{todonotes}
\newcommandx{\Maria}[2][1=]{\todo[linecolor=red,backgroundcolor=red!25,bordercolor=red,#1]{#2}}
\newcommandx{\Aleksis}[2][1=]{\todo[linecolor=blue,backgroundcolor=blue!25,bordercolor=blue,#1]{#2}}

\icmltitlerunning{Flexible SE(2) graph neural networks with applications to PDE surrogates}

\begin{document}

\twocolumn[
\icmltitle{Flexible SE(2) graph neural networks with applications to PDE surrogates}

\icmlsetsymbol{equal}{*}

\begin{icmlauthorlist}
\icmlauthor{Maria Bånkestad}{rise,uu}
\icmlauthor{Olof Mogren}{rise,climes}
\icmlauthor{Aleksis Pirinen}{rise}
\end{icmlauthorlist}

\icmlaffiliation{rise}{RISE Research Institutes of Sweden}
\icmlaffiliation{uu}{Uppsala University}
\icmlaffiliation{climes}{Swedish Centre for Impacts of Climate Extremes (climes)}

\icmlcorrespondingauthor{Maria Bånkestad}{maria.bankestad@ri.se}

\icmlkeywords{geometric deep learning, PDE surrogates}

\vskip 0.3in
]

\printAffiliationsAndNotice{}

\begin{abstract}
This paper presents a novel approach for constructing graph neural networks equivariant to 2D rotations and translations and leveraging them as PDE surrogates on non-gridded domains. We show that aligning the representations with the principal axis allows us to sidestep many constraints while preserving SE(2) equivariance.  By applying our model as a surrogate for fluid flow simulations and conducting thorough benchmarks against non-equivariant models, we demonstrate significant gains in terms of both data efficiency and accuracy. Code is available at \url{https://github.com/mariabankestad/SE2-GNN}.
\end{abstract}

\section{Introduction}\label{introduction}

Accurate and efficient numerical simulations of partial differential equations (PDEs) are important across various scientific disciplines, such as chemistry \citep{van1990computer}, physics \citep{ferziger2019computational}, and environmental sciences \citep{palmer2019stochastic}. Over the years, extensive research has produced diverse numerical methods for simulating PDEs, including finite difference, element, and volume methods \citep{peiro2005finite}. However, conducting high-fidelity simulations requires substantial computational resources and time, especially for complex systems or large-scale problems. As a result, there has been a growing interest in leveraging machine learning (ML) techniques to approximate these simulations more efficiently \citep{lavin2021simulation}. Surrogate ML models aim to approximate the behavior of PDE simulations with significantly lower computational costs while maintaining reasonable accuracy \citep{li2020neural, gupta2022towards, gladstone2024mesh}. 

Our focus is on problems in two spatial dimensions. Fluid simulations are frequently conducted in two dimensions since reducing problems from three to two dimensions simplifies the PDE and reduces computational costs. Moreover, many natural phenomena exhibit inherent two-dimensional symmetries, so limiting the simulation to two dimensions is natural. Moreover, PDEs often encode physical principles reflecting the fundamental symmetries of the systems they describe. For instance, the 2D Navier-Stokes equations (see Section~\ref{sec:navier-eqs}), essential in fluid dynamics simulations, exhibit inherent equivariance under the SE(2) group, meaning they retain their structure under rotations and translations in two-dimensional space. 

Irregular grids are commonly used in real-world fluid simulations to capture complex flows, such as flows around complex geometries. Therefore, developing models that work well on these grids is crucial for accurately representing real-world phenomena. Graph neural networks (GNNs), representing the computational domain as a graph, excel in such scenarios.
GNNs can also adapt to different grid resolution levels, making it possible to, for example, have higher resolutions in regions with high flow complexity and lower resolution where the complexity is low. This makes GNNs suitable as PDE surrogates for many real-world scenarios. 
\looseness=-1

\begin{figure}[t]
\vskip 0.25in
\begin{center}
\centerline{\includegraphics[width=0.95\linewidth]{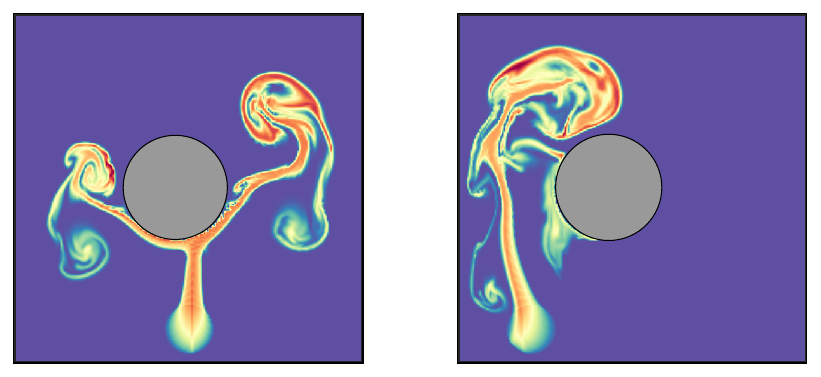}}
\caption{A snapshot from two simulations (with an upwards facing force) of smoke flowing around an obstacle, where the location of the smoke inflow differs.}
\label{fig:flow}
\end{center}
\vskip -0.1in
\end{figure}

We present a 2D equivariant GNN tailored for fluid flow simulations on irregular grids, which leverages equivariance to improve accuracy and efficiency. Our GNN is specifically tailored to address 2D problems equivariant to the SE(2) group. Focusing on SE(2) and two dimensions offers distinct advantages since we can simplify the model compared to alternatives that accommodate higher dimensions or more complex symmetry groups, thereby ensuring scalability. By leveraging the projection of 2D data onto a specific axis within the plane, we reduce the equivariance requirement from SE(2) to SE(1), thereby allowing for arbitrary nonlinear functions and message-passing convolutions. This, in turn, enables the incorporation of arbitrary message-passing layers, which enhances model flexibility. Experiments (Section~\ref{sec:experiments}) on fluid flow simulations show a clear performance boost in terms of data efficiency and accuracy compared to the model's non-equivariant counterpart.

\section{Background}\label{background}
This section summarizes the essentials for understanding subsequent sections. We first define equivariance and then introduce the graph neural network (GNN) model. Following this, we also describe the fundamentals of the Navier-Stokes equation, a cornerstone in many of our experiments. Finally, we describe the concept of ML surrogate models.
\subsection{Equivariance}
Leveraging equivariance has the potential to improve data efficiency in machine learning. It ensures that a model's predictions remain consistent under transformations of the input data so that the model does not have to learn separate representations for each transformation. The equivariance property is particularly important in tasks where the data exhibits symmetries, such as rotation or translation. 

A function $ f: X \to X$ is considered equivariant with respect to a group $G$ if it satisfies the following condition:
\[ f(g \cdot x) = g \cdot f(x) \quad \text{for all} \quad g \in G. \]

In simpler terms, we can transform the data by $g \in G$ before or after we apply the function $f (x)$, and the outcome will be the same. 

\subsection{Graph neural networks}
A graph $\mathcal{G} = (\mathcal{V}, \mathcal{E})$ consists of nodes $i \in \mathcal{V}$ and edges $(i,j) \in \mathcal{E}\subseteq \mathcal{V}\times\mathcal{V}$, which define the relationships between the nodes $i$ and $j$. A graph neural network (GNN) comprises multiple message-passing layers. At each layer $k$, given a node feature $ \vx_i^k $ at node $ i$, its neighboring nodes $ \{ \vx_j^k: j \in \mathcal{N}(i)\}$
and edge features $ \ve_{ij}^k $ between node $i$ and its neighbors, the message-passing procedure is defined as follows:
\begin{align}\label{eq:message_passing}
    \vm_{ij}^k &= f^\text{m}(\vx_i^k, \nonumber \vx_j^k, \ve_{ij}^k), \\
    {\vx}_i^{k+1/2} &= f^\text{a}_{j \in \mathcal{N}(i)} (\vm_{ij}^k),\\
    \vx_i^{k+1} &= f^\text{u}(\vx_i^k, {\vx}_i^{k+1/2}), \nonumber
\end{align}
where $f^\text{m} $ is the message function, determining the message from node $ j $ to node $ i $, and $f^\text{a}_{j \in \mathcal{N}(i)} $ aggregates messages from the neighbors of node $ i $, denoted $ \mathcal{N}(i) $. The aggregation function $ f^\text{a}$ often involves a simple summation or averaging. Finally, \( f^\text{u} \) is the update function that modifies the features for each node. These message-passing layers are stacked in a GNN, where the output from one layer serves as the input to the next.

\subsection{Navier-Stokes equations}\label{sec:navier-eqs}
The Navier-Stokes equations, a fundamental PDE, are a cornerstone of fluid dynamics. The equations describe fluid flow behaviors in various physical systems and are derived from fundamental principles such as mass and momentum conservation. Due to their versatility, they serve as a powerful tool for understanding complex fluid phenomena across a wide range of disciplines, from engineering to meteorology.

The incompressible Navier-Stokes equations \citep{chung2002computational} are given by:
\begin{align}\label{eq:navier}
    &\frac{\partial \vv}{\partial t} = -\overbrace{ \ \vv \cdot \nabla \vv \ }^{\parbox{1.2cm}{ \footnotesize Convection \\}} + \overbrace{\nu \nabla^2 \vv}^{\parbox{1.2cm}{ \footnotesize Viscosity \\}} -\overbrace{ \ \frac{1}{\rho}\nabla p \ }^{\parbox{1.2cm}{ \footnotesize Internal pressure \\}}  + \overbrace{\vf,}^{\parbox{1.2cm}{ \footnotesize External force \\}}\nonumber \\
    &\nabla \cdot \vv = 0 . 
\end{align}
Here, $\vv$ is the velocity vector field of the fluid, $\rho$ is the density of the fluid, $\mu$ is its viscosity (a measure of a fluid's resistance to flow), and $\vf$ represents any external forces acting on the fluid (e.g.~gravity).

The Navier-Stokes equations are nonlinear, meaning traditional linear techniques are insufficient to solve them; instead, more advanced methods are used, which approximate the solution using numerical methods \citep{griebel1998numerical}. When limiting the problem to 2D, the equations are inherently SE(2) equivariant. In other words, if we rotate or translate the coordinate system, the equations remain valid in the new coordinate system, and the solution will rotate or translate accordingly.

\subsection{PDE surrogate models}
We focus on time-dependent PDEs, where we want to approximate the solution at the next timestep, given the current one -- thus, we are looking for a surrogate for a traditional numerical solver. For example, consider the Navier-Stokes equations \eqq{eq:navier} with some domain parameters $\vd$. Then, we seek an ML model that can approximate the velocity field at the next timestep  $\vv^{t+1}$, knowing the current one $\vv^t$. We can express this as:
\begin{equation}
    \vv^{t+1} = f_\eta(\vv^{t}; \vd),
\end{equation}
where $f_\eta$ with parameters $\eta$ is our surrogate model.

For simplicity, the domain is typically discretized into a grid or mesh. This grid consists of discrete points indexed by $i$ and located at $\vr_i$ and their connections to their neighbors. This discretization naturally lends itself to a graphical representation, where nodes correspond to grid points and edges their connections. Leveraging this graph structure, we can employ GNNs as our surrogate model.

\section{SE(2) graph neural network}\label{sec:method}

This section outlines the main components of our SE(2) graph neural network (GNN), designed to be equivariant to the SE(2) group and capable of handling both scalars and vectors as inputs and outputs. Our model operates on a data structure represented as a graph, which includes node features $\vx_i$, edges $(i,j)$, and node locations $\vr_i$, all situated in a 2D plane. This graph could represent various scenarios, such as a snapshot of a fluid flow simulation. 

\textbf{SO(2) equivariance.}
To simplify the model's requirements to only being equivariant under SO(2), we conduct two actions: First, we change the center of mass for the nodes to zero, ensuring the graph's global translation equivariance. Second, we exclusively consider relative distances $\hat{\vr}_{ij} = {\vr}_{j}- {\vr}_{i}$ between nodes and their neighbors rather than the absolute coordinates in the message-passing modules. This adjustment ensures local translation equivariance. Now, the requirement for the different modules is reduced to only SO(2) equivariance. In the remainder of this section, we assume that the center of mass of the nodes in the graph is zero, and we will frequently utilize the concept of SO(2) equivariance instead of SE(2) equivariance.

The most basic component of our model is the rotation matrix $\boldsymbol{R}_\theta$, a transformation matrix used to perform a rotation in 2D space (Equation~\ref{eq:rot}). 

\begin{equation}
\vcenter{\hbox{\includegraphics[width=2.5cm,height=2.5cm]{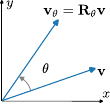}}}
\qquad
\begin{aligned}\label{eq:rot}
\mR_{\theta} = \begin{bmatrix}
    \cos \theta & -\sin \theta\\
    \sin \theta & \cos \theta
\end{bmatrix}
\end{aligned}
\end{equation}

The rotation matrix rotates vectors $\vv \in \mathbb{R}^2$ by the angle $\theta$ into $\vv_\theta$, i.e.~\mbox{$\vv_\theta= \mathbb{R}_\theta \vv$}. Throughout this section, we use $\theta$ for local rotation angles (around neighboring nodes) and $\alpha$ for global rotation angles.
\begin{figure*}[t]
\vskip 0.2in
\begin{center}
\centerline{\includegraphics[width=0.9999\linewidth]{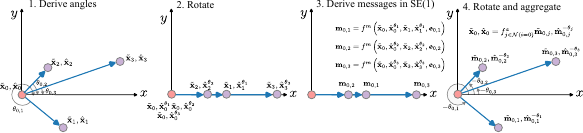}}
\caption{\textbf{The message-passing procedure.} We rotate the node features by the angles $\theta_i$ so that they all align with the $x$-coordinate axis. The scalar features $\tilde{\vx}_i$ and the rotation feature $\hat{\vx}_i$ can now be concatenated and inputted to any arbitrary message-passing function $f^m$. The messages are then rotated back and aggregated by arbitrary aggregation function $f^a$.}
\label{fig:mess_explained}
\end{center}
\vskip -0.1in
\end{figure*}
Assume that we want to apply a function $f$ on a node feature $\vx_i$ with position $\vr_i$ in the 2D plane and that this function is equivariant, such that $f(\mR \boldsymbol{x}_i, \mR \boldsymbol{r}_i) = \mR f(\boldsymbol{x}_i, \boldsymbol{r}_i)$ for all $\mR \in \text{SO(2)}$. We aim to lessen the equivariance constraint on $f$ by aligning all nodes along a shared axis, as doing so reduces the problem's symmetry to SO(1), which is nothing but the identity. We can, therefore, align the nodes to the $x$-axis, apply any nonlinear function, and then rotate the nodes back, and the equivariance criteria will still hold. We thus have the following proposition.

\begin{proposition}\label{prop:equivariance}
If $\mR_{\alpha_i}$ is a rotation to align $\vx_i, \vr_i$ to the $x$-axis, where $\alpha_i$ is the rotation angle, and $\mR_{-\alpha_i}$ is the rotation matrix inverse. Then, for a function
\begin{equation}
    f = \mR_{-\alpha_i} \circ g \circ \mR_{\alpha_i},
\end{equation} where $g$ is any nonlinear function, it holds that
\begin{equation}
    \mR  f\left( \vx_i ,\vr_i\right) = f\left (\mR\vx_i,\mR\vr_i \right), \quad \forall \mR \in SO(2).
\end{equation}
Thus, $f$ is equivariant under the SO(2) group.
\end{proposition}
\begin{proof}
    See Appendix \ref{app:proof_equivariance} for proof of this proposition.
\end{proof}


\textbf{Rotational features.}
We want our model to use scalar and vectors as inputs and outputs. We therefore separate the node features into scalar features $\tilde{\vx}_i$ and rotational features $\hat{\vx}_i$. Rotational features represent all features that have a direction in space and are composed of $N_r$ vectors in $\mathbb{R}^2$ stacked together:
\begin{equation}
    \hat{\vx}_i =  \hat{\vx}_i^1\oplus \hdots \oplus  \hat{\vx}^{N_r}_i.
\end{equation}
Here $\oplus$ stands for concatenating. When we apply a rotation $\mR_{\alpha_i}$ to $\hat{\vx}_i$ we individually apply the rotation to each $\hat{\vx}_i^n$, which results in:
\begin{equation}
    \mR_{\alpha_i}\hat{\vx}_i =  \mR_{\alpha_i}\hat{\vx}_i^1\oplus \hdots \oplus  \mR_{\alpha_i}\hat{\vx}^{N_r}_i.
\end{equation}
To simplify the notation, we henceforth use $\mR_{\alpha_i}\hat{\vx}_i$ to describe a rotation.

\textbf{Rotational MLP.}
Using Proposition \ref{prop:equivariance}, we define an SO(2) equivariant multi-layer perceptron (MLP), denoted \mlp{}:
\begin{align}\label{eq:rot_mlp}
    \tilde{\vx}_i \oplus \hat{\vx}_i^{r, in} &= \tilde{\vx}_i^\text{in} \oplus \mR_{\alpha_i} \hat{\vx}_i^\text{in}, \nonumber \\
    \tilde{\vx}_i^\text{out}\oplus \hat{\vx}_i^{r,out} &= \text{MLP}( \tilde{\vx}_i \oplus \hat{\vx}_i^{r, in}), \\
     \tilde{\vx}_i^\text{out} \oplus \hat{\vx}_i^\text{out} &= \tilde{\vx}_i^\text{out} \oplus \mR_{-\alpha_i} \hat{\vx}_i^{r,out}. \nonumber
\end{align}
Here, the input to the MLP is $\tilde{\vx}_i$ and $\hat{\vx}_i^{r, in}$ to enhance flexibility. It is worth noting that the dimensions of the inputs $\tilde{\vx}_i^\text{in}, \hat{\vx}_i^\text{in}$ and outputs $\tilde{\vx}_i^\text{out}, \hat{\vx}_i^\text{out}$ can differ. The only requirement is that the dimension of the rotational features must be a multiple of two since that is the dimension of vector (or representation) in two dimensions.

We will next describe the key components of our GNN. As a first step, we embed the input so that the model can process it. We assume that the input comprises \emph{scalar node features} $\tilde{\vx}_i^0$, which could represent scalar fields, masses, and so on; \emph{vector-valued node features} $\hat{\vx}_i^0$, such as velocities, forces, boundary vectors, and so on; and \emph{coordinate positions} $\vr_i$. 

\textbf{Edge embedding.}
To obtain a representation of the position $\vr_i$ of node $i$, we first consider the relative position between node $i$ and its neighbors $j \in \mathcal{N}(i)$:
\begin{equation}
    \vec{\vr}_{ij} = \vr_j - \vr_i, \quad  r_{ij} = \left \| \vec{\vr}_{ij} \right \|_2, \quad\hat{\vr}_{ij} = \frac{\vec{\vr}_{ij}}{r_{ij}}, \nonumber
\end{equation}
where $r_{ij}$ is the scalar distance and $\hat{\vr}_{ij}$ is the orientation vector between the nodes. To obtain a more expressive representation of the scalar distance $r_{ij}$, we embed it as a projection onto a radial basis:
\begin{equation}
    b(r_{ij}) = b_1(r_{ij}) \oplus \hdots \oplus b_{N_\text{base}}(r_{ij}),
\end{equation}
where \(N_\text{base}\) is the number of basis functions. We use the Bessel basis function among the various available options. The orientation information is encoded, using \eqq{eq:rot}, in the relative rotation matrix between neighboring nodes, 
\begin{equation}\label{eq:rot_loc}
    \theta_{ij} = -\atantwo( \hat{\vr}_{ij}), \ \ \
        \mR_{\theta_{ij}} = \begin{bmatrix}
    \cos \theta_{ij} & -\sin \theta_{ij}\\
    \sin \theta_{ij} & \cos \theta_{ij}.
\end{bmatrix}.
\end{equation}
This orientation matrix $\mR_{\theta_{ij}}$ ensures equivariance locally in the message-passing layer between neighbors. We must also ensure equivariance globally, so we also derive the nodes' global orientation matrix: 
\begin{equation}\label{eq:rot_glob}
    \alpha_i = -\atantwo\left ( \hat{\vr}_i\right), \ \ \ \ \ \
    \mR_{\alpha_i} = \begin{bmatrix}
    \cos \alpha_i & -\sin \alpha_i\\
    \sin \alpha_i & \cos \alpha_i
\end{bmatrix}.
\end{equation}
Here the angle $\alpha_i$ is the global angle between the position vectors $\hat{\vr}_i$ and the global positive $x$-axis

\textbf{Node embedding.}
The node input to the network consists of 
scalar node features $\tilde{\vx}_i^0$ and vector-valued node features $\hat{\vx}_i^0$. We get the node representation by using:
\begin{equation}
    \tilde{\vx}_i^1 = \text{MLP}( \tilde{\vx}_i^0), \quad  \hat{\vx}_i^1 = \text{SO2-MLP}( \hat{\vx}_i^0, \alpha_i),
\end{equation}
where the rotational embedding $\hat{\vx}_i^1$ is obtained using \eqq{eq:rot_mlp}, without any scalar features.

\textbf{Message-passing layer.}
We are now prepared to construct the message-passing layer, which serves as the central building block of a GNN, facilitating interactions between nodes through the message-passing function $f^m$. 
An overview of the message-passing steps is illustrated in Figure~\ref{fig:mess_explained}.\looseness=-1

The inputs to the message-passing layer include the node features $\vx_i^k$, the features of neighboring nodes $ \boldsymbol{x}_j^k $, the edge distance embedding $ \vb_{ij}$, and the local rotational matrix $ \boldsymbol{R}_{\theta_{ij}}$ (refer to \eqq{eq:rot_loc}). The message-passing layer proceeds through the following steps:

\begin{align}
    \tilde{\vm}_{ij}^{k} \oplus \hat{\vm}_{ij}^{r,k}  &= \tilde{\vx}_i \oplus \tilde{\vx}_j \oplus 
                \vb_{ij} \oplus
                \boldsymbol{R}_{\theta_{ij}}\hat{\vx}_i\oplus \boldsymbol{R}_{\theta_{ij}}  \hat{\vx}_j, \nonumber\\
   \tilde{\vm}_{ij}^{k+1} \oplus \hat{\vm}_{ij}^{r,k+1} &= f^m(\tilde{\vm}_{ij}^{k} \oplus \hat{\vm}_{ij}^{r,k} ), \nonumber \\
   \tilde{\vm}_{ij}^{k+1} \oplus \hat{\vm}_{ij}^{k+1} &= \tilde{\vm}_{ij}^{k+1}\oplus \boldsymbol{R}_{-\theta_{ij}} \hat{\vm}_{ij}^{r,k+1}, \nonumber\\
   \tilde{\vx}_{ij}^{k+1} \oplus \hat{\vx}_{ij}^{k+1} &= f^a_{j\in \mathcal{N}(i)}\left ( \tilde{\vm}_{ij}^{k+1} \oplus \hat{\vm}_{ij}^{k+1}\right).
\end{align}

Here, we concatenate the node features $i$ with its neighboring features $j\in \mathcal{N}(i)$ before passing them through the message-passing function $f^m$. We can employ arbitrary message-passing due to Proposition \ref{prop:equivariance}. Additionally, $ f^a $ represents any permutation-invariant aggregation function, with $f^a$ a summation in our case.

Even though we could potentially apply an arbitrary message-passing layer, we limit ourselves to investigating two: one inspired by a graph Transformer, using self-attention, and one using MLPs in the message function, similar to the one used in \citep{brandstetter2022message}. These are respectively defined next.

\textbf{SE2Conv-MLP.} This operation is given simply by:
 \begin{equation}\label{eq:mess_mlp}
    f^m = \text{MLP}(\vm_{ij})
\end{equation}
\textbf{SE2Conv-Trans.} This operation is defined by the following sub-operations:
\begin{align}\label{eq:self_attention}
     f^m &= \text{Linear}\left( \text{LeakyReLU}\left(\vz_{ij}\right)\right), \text{where} \nonumber \\  
     \alpha_{ij} &= \text{softmax} \left(\frac{\text{Linear} \left(\text{LeakyReLU} \left(\text{LN}\left(\vz_{ij} \right) \right)\right)}{\sqrt{d}}\right)\nonumber\\
     \vz_{ij} &= \text{Linear} \left ( \vm_{ij}\right ),  
\end{align}
where LN stands for layer normalization \citep{ba2016layer}.
\begin{figure}[t]
\vskip 0.2in
\begin{center}
\centerline{\includegraphics[width=0.65\linewidth]{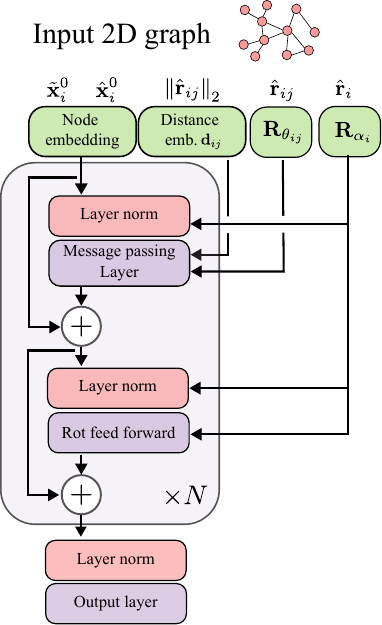}}
\caption{\textbf{An outline of our model.} The top green blocks denote the input embedding of the data, followed by the main block that comprises a message-passing and a feed-forward layer. Finally, the output layer aligns with what we aim to predict.}
\label{fig:model_overview}
\end{center}
\vskip -0.2in
\end{figure}
\textbf{Feed-forward block.} 
A feed-forward layer is applied between the message-passing layer. To do this, we use the \mlp{} in ~\eqq{eq:rot_mlp}, where the input and output dimensions of the node features are equal.

\textbf{Layer norm.} 
For layer normalization, we ensure equivariance by employing the separable layer norm detailed in \citep{liao2023equiformerv2}, adapted to two dimensions, with the maximum degree set to one, similar to $L_{\text{max}}$ in \citep{liao2023equiformerv2}. Essentially, this means that we independently normalize scalar features $ \tilde{\vx}_i^k $ and rotation features $ \hat{\vx}_i^k $. For scalar features, the normalization is defined as:

\begin{equation}
    \tilde{\vx}^\text{out}_i = \boldsymbol{\gamma}^s \circ \left ( \frac{\tilde{\vx}^\text{in}_i-{{\mu}}_s}{{\sigma}_s} \right ) .
\end{equation}

Here, ${\gamma}^s$ (subscript $s$ for scalars) is a vector with the same dimension as $\tilde{\vx}_i$, $\circ$ is the element-wise multiplication, $ {\mu}_s = \frac{1}{C} \sum_{c=1}^C \tilde{x}^\text{in}_{i,c}$ is the mean, and \mbox{${\sigma}_s = \frac{1}{C} \sum_{c=1}^C \left (\tilde{x}^\text{in}_{i,c}- {\mu}_s \right )^2$} the standard deviation calculated across channels $C$. On the other hand, for rotational features, the normalization is performed as:
\begin{equation}
    \hat{\vx}^\text{out}_i = \frac{{\gamma}_r^1\hat{\vx}_i^{\text{in},1}}{{\sigma}_r}\oplus \hdots \oplus  \frac{{\gamma}_r^{N_r}\hat{\vx}_i^{\text{in},N_r}}{{\sigma}_r}.
\end{equation}
Here, ${\gamma}^i_r$ (subscript $r$ for rotation) are scalars and $\sigma_r = \sqrt{ \frac{1}{C} \sum_{c=1}^C  \frac{1}{2}\sum_{m=1}^2\left (\hat{x}^\text{in}_{m,c,i} \right )^2 }$, where we also average over the two parts of the rotation representation.  This ensures equivariance since sum and scalar multiplication are equivariant operations in SO(2).

\textbf{Output layer.} 
Our model provides a straightforward approach to predicting either a scalar feature, such as the value of a scalar field or energy, or a rotational feature, such as velocity, force, or movement in space. We again utilize the \mlp{} (cf.~\eqq{eq:rot_mlp}) for this purpose. For predicting a scalar output, if $\tilde{\vx}_i^{K} \oplus \hat{\vx}_i^{K,r}$ represents the output after $K$ layers of message-passing, the scalar output is computed as:
\begin{align}\label{eq:out_scalar}
    \tilde{\vx}_i^{K} \oplus \hat{\vx}_i^{K,r} &= \tilde{\vx}_i^{K} \oplus \mR_{\alpha_i} \hat{\vx}_i^{K}, \nonumber \\
     {\vx}_i^{s,\text{out}} &= \text{MLP}( \tilde{\vx}_i^K \oplus \hat{\vx}_i^{K,r}).
\end{align}
For the rotational output, the process is:
\begin{align}\label{eq:out_rot}
    \tilde{\vx}_i^K \oplus \hat{\vx}_i^{K,r} &= \tilde{\vx}_i^{K} \oplus \mR_{\alpha_i} \hat{\vx}_i^{K}, \nonumber \\
     \hat{\vx}_i^{\text{out},r} &= \text{MLP}( \tilde{\vx}_i^K \oplus \hat{\vx}_i^{K,r}), \\
     \hat{\vx}_i^\text{out} &=  \mR_{-\alpha_i} \hat{\vx}_i^{\text{out},r}. \nonumber
\end{align}
Here, ${\vx}_i^{r,out}$ represents the network's rotational output, such as a vector. A pooling layer can be applied if graph prediction is the objective.

We construct a deep SE(2) GNN, outlined in Figure~\ref{fig:model_overview}, inspired by the block structure of the Transformer model \citep{vaswani2017attention}. Each block in our model consists of an SE(2) message-passing, an SE(2) feed-forward layer, and SE(2) layer norms. Additional details of our method are in Appendix \ref{add:method}.

\section{Experiments}\label{sec:experiments}
In this section, we leverage the SE(2) GNN outlined in Section \ref{sec:method} across different experiments. While our primary focus is on developing surrogate models for 2D fluid simulation, we start with a preliminary experiment to demonstrate the equivariance of our model and illustrate the significance of this property. Subsequently, we delve into experiments that revolve around the fluid simulation problem. Additionally, in Appendix \ref{app:additional_navier_expermint}, we include a small experiment comparing our model to another type of SO(2) message passing that does not use projection.

\begin{figure}[t]
\vskip 0.2in
\begin{center}
\centerline{\includegraphics[width=0.6\columnwidth]{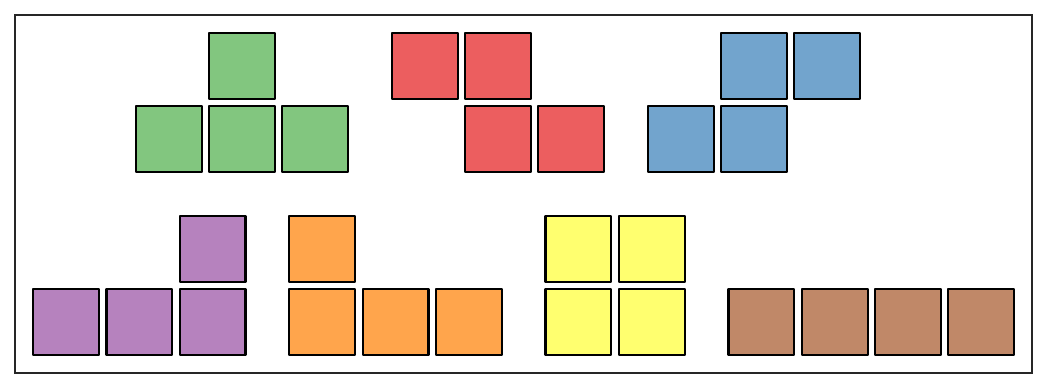}}
\caption{The seven shapes in the Tetris training dataset. The tests will of these shapes rotated at random angles $\theta$ (see Figure \ref{fig:tetris_test_blocks} in the appendix).}
\label{fig:tetris_train}
\end{center}
\vskip -0.2in
\end{figure}

\begin{figure}[t]
\vskip 0.2in
     \centering
     \begin{subfigure}
         \centering
         \includegraphics[width=0.9\linewidth]{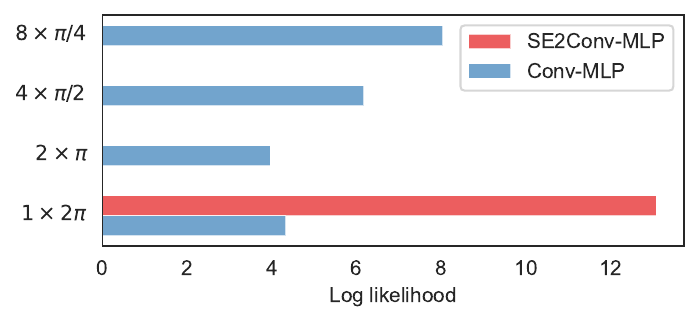}
     \end{subfigure}
     \vfill
     \begin{subfigure}
         \centering
         \includegraphics[width=0.9\linewidth]{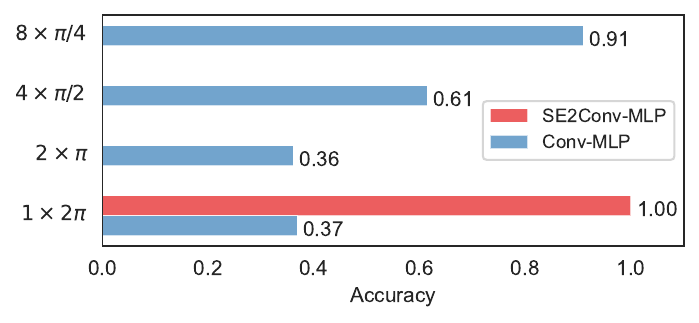}
     \end{subfigure}
        \caption{Our SE(2) model obtains perfect accuracy despite being trained only a single copy of each Tetris shape, whereas the regular message-passing network obtains similar results only after being exposed to eight copies of each shape.}
        \label{fig:tetris_results}
\vskip -0.in
\end{figure}

\subsection{Equivariance evaluation using 2D Tetris classification}

To showcase the SE2-network's capabilities, we consider the task of classifying the seven shapes in 2D Tetris, as depicted in Figure \ref{fig:tetris_train}. The goal is to classify these shapes when they undergo random rotations by an angle $\beta$ from their initial orientation. Each Tetris shape is described by its coordinates $\vr_i$. To learn to classify these shapes, we generate four distinct training datasets outlined in Table~\ref{tab:tetris_dataset}. 

\begin{table}[ht]
\caption{The four different datasets for learning to classify the rotated Tetris shapes.}
\label{tab:tetris_dataset}
\begin{center}
\begin{footnotesize}
\begin{sc}
\begin{tabular}{lcccr}
\toprule
Dataset &\mbox{Rotation ang.}&\# per shape & \# tot \\
\midrule
 $1\times 2\pi$    & $2\pi$& 1&7 \\
 $2\times \pi$  & $\pi$&2& 14\\
 $4\times \pi/2$ & $\pi/2$ &4&28 \\
 $8\times \pi/4$ & $\pi/4$ &8&   56      \\
\bottomrule
\end{tabular}
\end{sc}
\end{footnotesize}
\end{center}
\vskip -0.1in
\end{table}
Except for the original dataset with the seven shapes, we employ data augmentation techniques by rotating the Tetris shapes at varying angles, including $\pi$, $\pi/2$, and $\pi/4$. We also create a test set with the seven shapes, randomly rotated 100 times each, resulting in a test set of 700 shapes.

Our model consists of a two-layer SE(2) message-passing neural network (as detailed in Section \ref{sec:method}), which utilizes the message-passing layer \eqq{eq:mess_mlp}. Creating an input node embedding poses a slight challenge since we only have node positions and aim to maintain translational equivariance. Therefore, we create an embedding of the nodes using the relative distance vector $\hat{\vr}_{ij}$, expressed as:

\begin{equation} \tilde{\vx}_i^1\oplus \hat{\vx}_i^1 =\sum_{j\in \mathcal{N}(i)} \text{SO2-MLP}( \hat{\vr}_{ij}, \theta_{ij}).\end{equation}

The network's output is mapped to a single scalar using the output layer \eqq{eq:out_scalar}, with an output dimension of one. Subsequently, a pooling layer sums the node outputs of the graph. We also develop an invariant model version, which mirrors the SE(2) model but excludes rotations. Both models are trained by minimizing the cross entropy loss during 100 epochs using the Adam optimizer~\citep{kingma2014adam}, with a learning rate of $10^{-3}$. Additional details are found in Appendix \ref{app:tetris}.

The results in Figure \ref{fig:tetris_results} demonstrate that our SE(2) model achieves perfect accuracy even when trained on the smallest dataset with only the seven original shapes. In contrast, when augmented with eight rotations, which means 56 shapes in the training dataset, the non-rotational message-passing model achieves an accuracy of 0.96. Although the test accuracy eventually becomes high, the log-likelihood remains lower than that of our SE(2) model. Additional plots of the Tetris experiment are found in Appendix \ref{app:tetris}.

\subsection{Navier-Stokes simulation}\label{sec:navier-exp1}
We now shift attention to training simulation surrogates using our SE(2) GNN, specifically, solving the Navier-Stokes equations \eqq{eq:navier} of fluid dynamics. Alongside the velocity field, we introduce a scalar field representing smoke, which moves with the velocity field -- a process known as advection. However, the scalar field only influences the velocity field through an external buoyancy force term; this is called weak coupling. Our training data, akin to \citep{brandstetter2022clifford}, is generated on a grid with a spatial resolution of $128\times 128$, and a timestep of $\Delta t = 1.5$ seconds, utilizing $\Phi_{\text{FLOW}}$ for simulations \citep{holl2020phiflow}.

\begin{figure}[th]
\vskip 0.2in
\begin{center}
\centerline{\includegraphics[width=0.8\columnwidth]{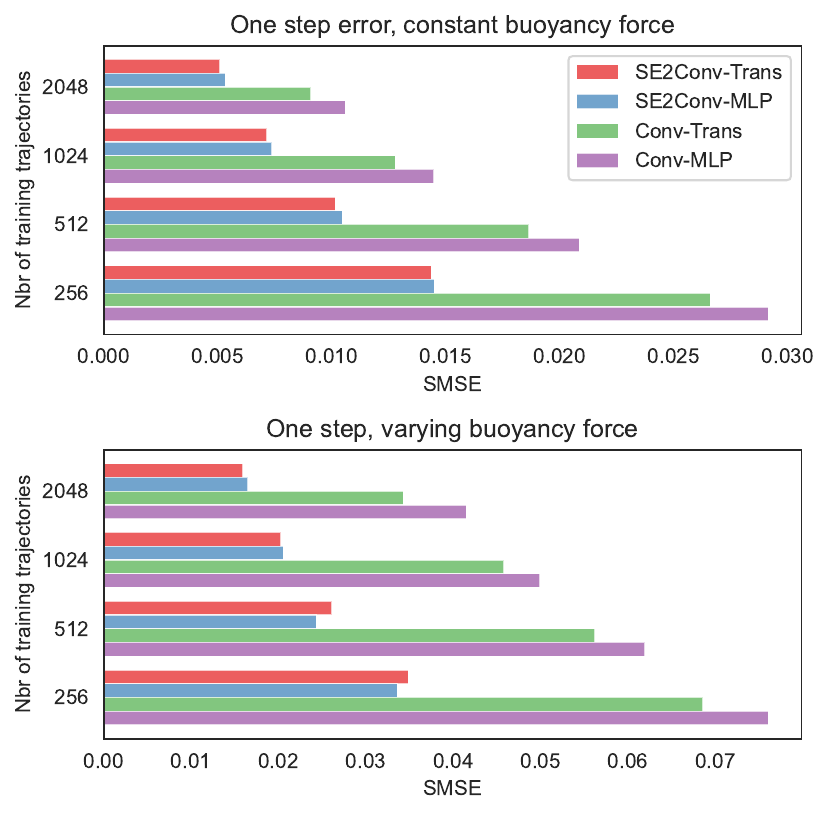}}
\caption{Our proposed approach is vastly more data-efficient than alternatives -- e.g. already at only 256 training trajectories, it matches the one-step error of TransConv at 2048 trajectories.}
\label{fig:navier_one_step}
\end{center}
\vskip -0.2in
\end{figure}
\begin{figure}[h]
\vskip 0.2in
\begin{center}
\centerline{\includegraphics[width=0.8\columnwidth]{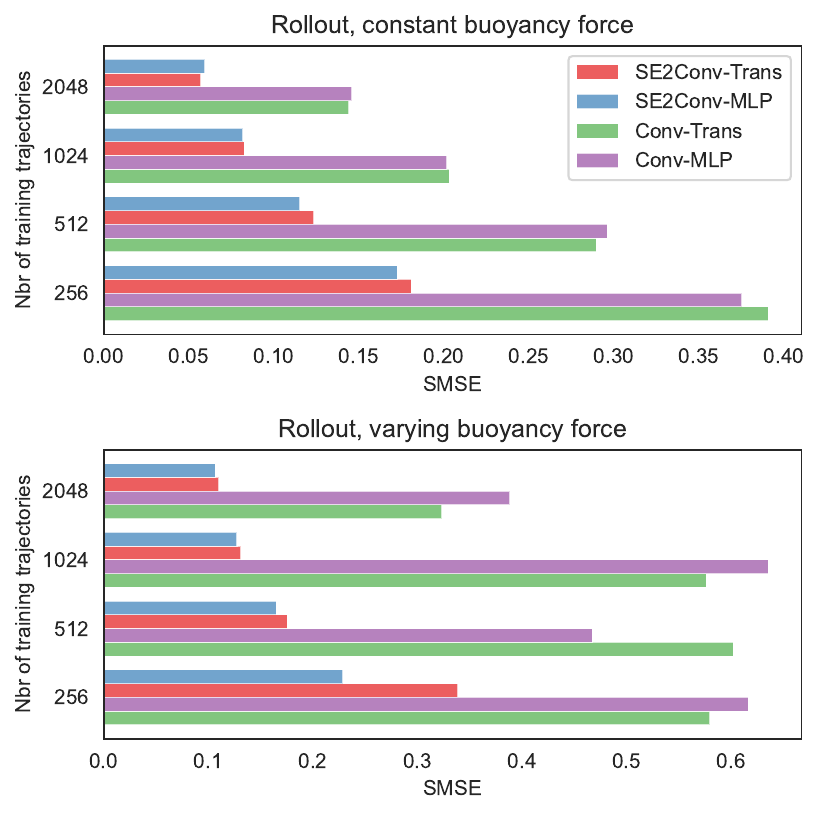}}
\caption{The improvements with respect to rollout errors are even more pronounced for our approach (cf. one-step errors in Figure~\ref{fig:error_one_step}), again showcasing the superior data efficiency of our approach.}
\label{fig:rollout}
\end{center}
\vskip -0.2in
\end{figure}

Following a similar approach to \citep{brandstetter2022clifford}, we incorporate a buoyancy force of $f=[0,0.5]$. However, we also create a new dataset by introducing variations in the buoyancy force. In particular, for each simulation trajectory, we randomly draw $f_x$ and $f_y$ from a continuous uniform distribution $f_x\sim\mathcal{U}(-0.7,7)$, defining the buoyancy force as $\vf = [f_x, f_y]$. For more on how the data is generated, we refer to \citep{brandstetter2022clifford}.

We aim to develop simulation surrogates for irregular domains, with accuracy obtained from a denser domain. We randomly sample $1024$ nodes from the regular grid for each simulated trajectory to create an irregular domain and then create the edges using Delaunay triangulation \citep{preparata2012computational}. Success in this endeavor would yield a model capable of operating on a sparse, irregular domain while maintaining the accuracy achieved by simulations on a dense square grid.

The objective is to forecast the scalar and vector fields at the next timestep based on a history of three timesteps. The number of timesteps is arbitrary and could be changed. The input data comprises node positions $\vr_i$, scalar field values $\vu_i$, and velocity field values $\vv_i$ for the previous three timesteps.
It also includes boundary normal vectors $\hat{\vn}_i$, which are zero for nodes not situated at the boundary, and a constant buoyancy force vector $\vf_i$ applied uniformly across the grid.

For this task, we develop two distinct seven-layer SE(2) GNN models: one utilizing SE2Conv-MLP, cf.~\eqq{eq:mess_mlp} and the other employing SE2Conv-Trans, cf.~\eqq{eq:self_attention}. Additionally, we create two non-equivariant counterparts that are identical to the SE(2) models, except that they lack rotational feature processing (model details are found in Appendix \ref{app:navier_extra}). We train each model by minimizing the summed mean square error, used in \cite{brandstetter2022clifford} between the predicted fields and the actual ones as
\begin{equation}\label{eq:smse_loss}
    \mathcal{L}_{SMSE} = \frac{1}{N_b}\sum_{i=1}^{N_b} (u_i- u_i^t)^2 + (v_{i,x}-v_{i,x}^t)^2 + (v_{i,y}-v_{i,y}^t)^2 
\end{equation}
where $N_b$ is the number of nodes in the batch, and superscript $t$ are the target values.

The experimental results, depicted in Figures \ref{fig:navier_one_step} and \ref{fig:rollout}, demonstrate the superiority of the SE(2) models over the naive non-equivariant counterpart across all scenarios. This performance gap widens notably when considering the rollout error (the accumulated prediction error over multiple timesteps). Notably, the advantages of employing SE(2) become even more pronounced in scenarios involving varying forces. 

\subsection{Navier-Stokes simulation with obstacle}
In this experiment, we investigate how well our proposed approach predicts the trajectory of incoming smoke that encounters an obstacle. The lineup of the simulations is similar to that in Section~\ref{sec:navier-exp1}, but instead of utilizing a random initial smoke, we have a smoke inlet where smoke seeps in at a constant rate. We also add an obstacle in the shape of a ball to complicate the simulation further -- see Figure~\ref{fig:flow}.   

We use the SE(2) message-passing model with the SE2Conv-Trans convolutional layer and its invariant counterpart (cf.~Section~\ref{sec:navier-exp1}), with the addition that we add a scalar feature representing the incoming smoke. We train the models similarly, using the SMSE loss \eqq{eq:smse_loss}. For additional information on the simulation, see Appendix \ref{add:navier_sim}.
\begin{figure}[t]
\vskip 0.2in
\begin{center}
\centerline{\includegraphics[width=0.95\columnwidth]{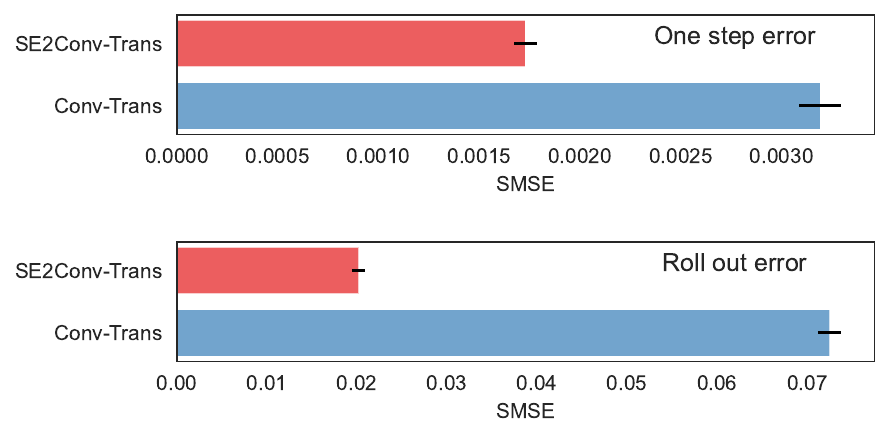}}
\vspace{-4pt}
\caption{Our approach (red) obtains significantly lower one-step as well as rollout errors compared to the non-invariant counterpart (blue). The rollout error is, however, even more reduced, which showcases the ability of our model to maintain accurate prediction accuracy even over longer time horizons.}
\label{fig:error_one_step}
\end{center}
\vskip -0.2in
\end{figure}

\begin{figure}[ht]
\vskip 0.2in
\begin{center}
\centerline{\includegraphics[width=0.75\columnwidth]{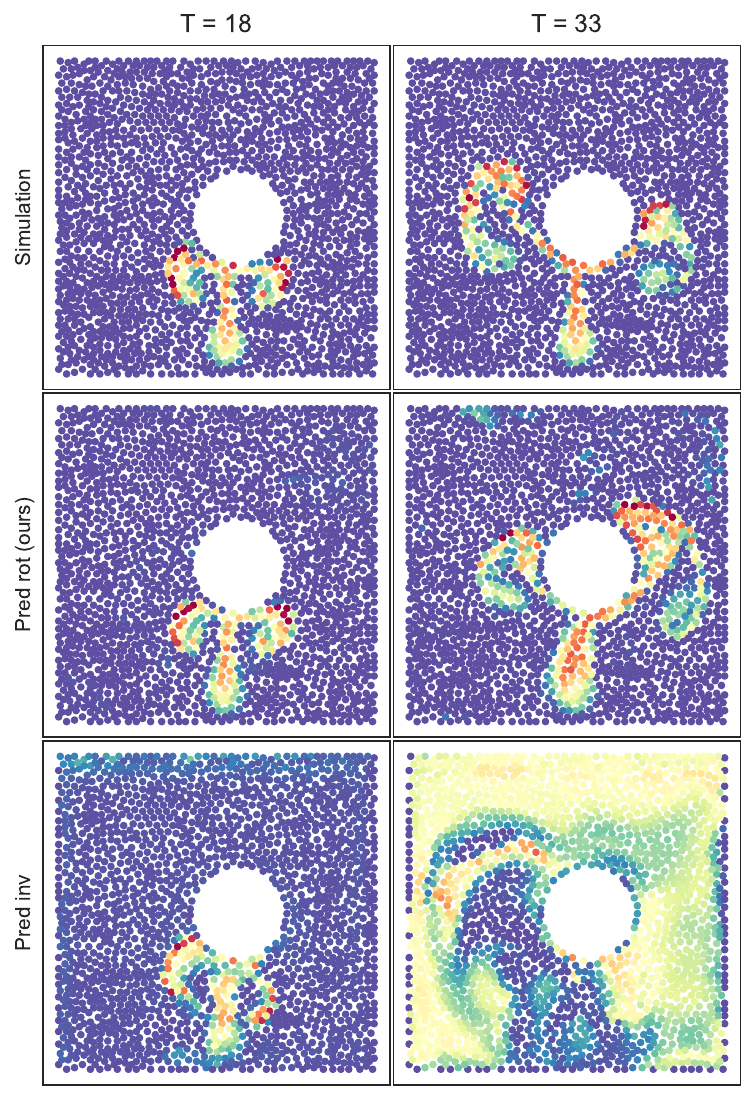}}
\caption{Projected evolution of a smoke encountering an obstacle. The models are fed with the first three simulated timesteps (first four seconds). We show the actual simulation (top row) and the respective model predictions at 18 seconds and 33 seconds. Our approach (middle) obtains significantly more accurate predictions even at longer time horizons. For a live visualization of this figure, see \url{https://github.com/mariabankestad/SE2-GNN}.}
\label{fig:roll_out_long_smoke}
\end{center}
\vskip -0.34in
\end{figure}

The results in Figures \ref{fig:error_one_step} and \ref{fig:roll_out_long_smoke} clearly show that our approach is well-suited for predicting the flows even in this more challenging setup. In particular, Figure~\ref{fig:error_one_step} showcases our approach's superior one-step and rollout prediction performance compared to the non-equivariant alternative (particularly for the rollout prediction). Figure~\ref{fig:roll_out_long_smoke}, in turn, provides a qualitative example where it is clear that our approach can maintain significantly more accurate predictions even at longer time horizons.

\section{Related work}

We divide the related work into two sections: first, we discuss surrogate models for fluid simulations, then 2D equivariant models and their counterparts in 3D.

\textbf{PDE surrogates.} 
Surrogate models can, roughly, be divided into three distinct categories:  Physics-informed neural networks (PINNs) \citep{ raissi2019physics, cai2021physics} are neural network models designed to incorporate physical laws and constraints into the training process; Neural operators \citep{li2020neural}, such as Fourier neural operators \citep{li2021fourier, kovachki2023neural} and Deep Operator Networks \citep{goswami2022physics}, extend the ML framework to learn mappings between infinite-dimensional spaces;  
and ML/neural PDE surrogates \citep{zhu2018bayesian, iakovlev2020learning, belbute2020combining,gladstone2024mesh, scalabelTransformLi,hemmasian2024multi} that focus on approximating solutions to specific partial differential equations using neural networks, sometimes a GNN \citep{sanchez2020learning,brandstetter2022message}, trained on data generated from PDE simulations or experiments. Our models fit best into the last category, even though it can adapt to variations in the input domain.

\textbf{Equivariant graph neural networks.}
Several efforts have been undertaken to develop PDE surrogates using equivariant models. \citet{horie2022physics} introduced a model equivariant to isometric transformations. \citet{brandstetter2022clifford} employed geometric algebra to create CNN-based PDE surrogates. \citet{lino2022multi} implemented a multi-level GNN for unsteady Eulerian fluid dynamics, exploiting equivariance by projecting velocities onto edge directions. Other endeavors have explored equivariant models, such as networks equivariant to rotation within the SO(3) group \citep{weiler20183d, geiger2022e3nn, baankestad2023carbohydrate} and CNNs equivariant to the SO(2) group \citep{worrall2017harmonic, weiler2019general, passaro2023reducing}. \citet{passaro2023reducing}, notably, introduced a faster SO(3) GNN by integrating SO(2) message-passing layers.

\section{Conclusions}
In this paper, we introduce a novel SE(2)-equivariant graph neural network (GNN) that exhibits very strong data efficiency and seamlessly operates in non-gridded domains. This was showcased on a proof-of-concept test bed of 2D Tetris shapes and real-world simulations of the Navier-Stokes equations. Our model, both scalable and flexible, significantly outperforms non-equivariant counterparts.

\textbf{Code availability.} Code is available at \url{https://github.com/mariabankestad/SE2-GNN}.

{\small \noindent{\bf Acknowledgments:} This work was funded by Vinnova, grant number 2023-01398. The simulations were performed on the Luxembourg national supercomputer MeluXina. The authors gratefully acknowledge the LuxProvide teams for their expert support.}

\bibliography{reference}

\begin{thebibliography}{42}
\providecommand{\natexlab}[1]{#1}
\providecommand{\url}[1]{\texttt{#1}}
\expandafter\ifx\csname urlstyle\endcsname\relax
  \providecommand{\doi}[1]{doi: #1}\else
  \providecommand{\doi}{doi: \begingroup \urlstyle{rm}\Url}\fi

\bibitem[Ba et~al.(2016)Ba, Kiros, and Hinton]{ba2016layer}
Ba, J.~L., Kiros, J.~R., and Hinton, G.~E.
\newblock Layer normalization.
\newblock \emph{arXiv preprint arXiv:1607.06450}, 2016.

\bibitem[B{\aa}nkestad et~al.(2023)B{\aa}nkestad, Dorst, Widmalm, and R{\"o}nnols]{baankestad2023carbohydrate}
B{\aa}nkestad, M., Dorst, K.~M., Widmalm, G., and R{\"o}nnols, J.
\newblock Carbohydrate nmr chemical shift predictions using e (3) equivariant graph neural networks.
\newblock \emph{arXiv preprint arXiv:2311.12657}, 2023.

\bibitem[Belbute-Peres et~al.(2020)Belbute-Peres, Economon, and Kolter]{belbute2020combining}
Belbute-Peres, F. D.~A., Economon, T., and Kolter, Z.
\newblock Combining differentiable pde solvers and graph neural networks for fluid flow prediction.
\newblock In \emph{international conference on machine learning}, pp.\  2402--2411. PMLR, 2020.

\bibitem[Brandstetter et~al.(2022{\natexlab{a}})Brandstetter, Berg, Welling, and Gupta]{brandstetter2022clifford}
Brandstetter, J., Berg, R. v.~d., Welling, M., and Gupta, J.~K.
\newblock Clifford neural layers for pde modeling.
\newblock \emph{arXiv preprint arXiv:2209.04934}, 2022{\natexlab{a}}.

\bibitem[Brandstetter et~al.(2022{\natexlab{b}})Brandstetter, Worrall, and Welling]{brandstetter2022message}
Brandstetter, J., Worrall, D., and Welling, M.
\newblock Message passing neural pde solvers.
\newblock \emph{The International Conference on Learning Representations, ICLR}, 2022{\natexlab{b}}.

\bibitem[Cai et~al.(2021)Cai, Mao, Wang, Yin, and Karniadakis]{cai2021physics}
Cai, S., Mao, Z., Wang, Z., Yin, M., and Karniadakis, G.~E.
\newblock Physics-informed neural networks (pinns) for fluid mechanics: A review.
\newblock \emph{Acta Mechanica Sinica}, 37\penalty0 (12):\penalty0 1727--1738, 2021.

\bibitem[Cesa et~al.(2022)Cesa, Lang, and Weiler]{cesa2022a}
Cesa, G., Lang, L., and Weiler, M.
\newblock A program to build {E(N)}-equivariant steerable {CNN}s.
\newblock In \emph{International Conference on Learning Representations}, 2022.
\newblock URL \url{https://openreview.net/forum?id=WE4qe9xlnQw}.

\bibitem[Chung(2002)]{chung2002computational}
Chung, T.~J.
\newblock \emph{Computational fluid dynamics}.
\newblock Cambridge university press, 2002.

\bibitem[de~Haan et~al.(2021)de~Haan, Weiler, Cohen, and Welling]{dehaan2021}
de~Haan, P., Weiler, M., Cohen, T., and Welling, M.
\newblock Gauge equivariant mesh cnns: Anisotropic convolutions on geometric graphs.
\newblock In \emph{International Conference on Learning Representations}, 2021.

\bibitem[Ferziger et~al.(2019)Ferziger, Peri{\'c}, and Street]{ferziger2019computational}
Ferziger, J.~H., Peri{\'c}, M., and Street, R.~L.
\newblock \emph{Computational methods for fluid dynamics}.
\newblock springer, 2019.

\bibitem[Fey \& Lenssen(2019)Fey and Lenssen]{FeyLenssen2019}
Fey, M. and Lenssen, J.~E.
\newblock Fast graph representation learning with {PyTorch Geometric}.
\newblock In \emph{ICLR Workshop on Representation Learning on Graphs and Manifolds}, 2019.

\bibitem[Geiger \& Smidt(2022)Geiger and Smidt]{geiger2022e3nn}
Geiger, M. and Smidt, T.
\newblock e3nn: Euclidean neural networks.
\newblock \emph{arXiv preprint arXiv:2207.09453}, 2022.

\bibitem[Gladstone et~al.(2024)Gladstone, Rahmani, Suryakumar, Meidani, D’Elia, and Zareei]{gladstone2024mesh}
Gladstone, R.~J., Rahmani, H., Suryakumar, V., Meidani, H., D’Elia, M., and Zareei, A.
\newblock Mesh-based gnn surrogates for time-independent pdes.
\newblock \emph{Scientific Reports}, 14\penalty0 (1):\penalty0 3394, 2024.

\bibitem[Goswami et~al.(2022)Goswami, Yin, Yu, and Karniadakis]{goswami2022physics}
Goswami, S., Yin, M., Yu, Y., and Karniadakis, G.~E.
\newblock A physics-informed variational deeponet for predicting crack path in quasi-brittle materials.
\newblock \emph{Computer Methods in Applied Mechanics and Engineering}, 391:\penalty0 114587, 2022.

\bibitem[Griebel et~al.(1998)Griebel, Dornseifer, and Neunhoeffer]{griebel1998numerical}
Griebel, M., Dornseifer, T., and Neunhoeffer, T.
\newblock \emph{Numerical simulation in fluid dynamics: a practical introduction}.
\newblock SIAM, 1998.

\bibitem[Gupta \& Brandstetter(2022)Gupta and Brandstetter]{gupta2022towards}
Gupta, J.~K. and Brandstetter, J.
\newblock Towards multi-spatiotemporal-scale generalized pde modeling.
\newblock \emph{arXiv preprint arXiv:2209.15616}, 2022.

\bibitem[Hemmasian \& Farimani(2024)Hemmasian and Farimani]{hemmasian2024multi}
Hemmasian, A. and Farimani, A.~B.
\newblock Multi-scale time-stepping of partial differential equations with transformers.
\newblock \emph{Computer Methods in Applied Mechanics and Engineering}, 426:\penalty0 116983, 2024.

\bibitem[Holl et~al.(2020)Holl, Koltun, Um, and Thuerey]{holl2020phiflow}
Holl, P., Koltun, V., Um, K., and Thuerey, N.
\newblock phiflow: A differentiable pde solving framework for deep learning via physical simulations.
\newblock In \emph{NeurIPS workshop}, volume~2, 2020.

\bibitem[Horie \& Mitsume(2022)Horie and Mitsume]{horie2022physics}
Horie, M. and Mitsume, N.
\newblock Physics-embedded neural networks: Graph neural pde solvers with mixed boundary conditions.
\newblock \emph{Advances in Neural Information Processing Systems}, 35:\penalty0 23218--23229, 2022.

\bibitem[Iakovlev et~al.(2021)Iakovlev, Heinonen, and L{\"a}hdesm{\"a}ki]{iakovlev2020learning}
Iakovlev, V., Heinonen, M., and L{\"a}hdesm{\"a}ki, H.
\newblock Learning continuous-time pdes from sparse data with graph neural networks.
\newblock \emph{International Conference on Learning Representations (ICLR)}, 2021.

\bibitem[Kingma \& Ba(2015{\natexlab{a}})Kingma and Ba]{KingBa15}
Kingma, D. and Ba, J.
\newblock Adam: A method for stochastic optimization.
\newblock In \emph{International Conference on Learning Representations (ICLR)}, San Diega, CA, USA, 2015{\natexlab{a}}.

\bibitem[Kingma \& Ba(2015{\natexlab{b}})Kingma and Ba]{kingma2014adam}
Kingma, D.~P. and Ba, J.
\newblock Adam: {A} method for stochastic optimization.
\newblock In \emph{ICLR}, 2015{\natexlab{b}}.

\bibitem[Kovachki et~al.(2023)Kovachki, Li, Liu, Azizzadenesheli, Bhattacharya, Stuart, and Anandkumar]{kovachki2023neural}
Kovachki, N., Li, Z., Liu, B., Azizzadenesheli, K., Bhattacharya, K., Stuart, A., and Anandkumar, A.
\newblock Neural operator: Learning maps between function spaces with applications to pdes.
\newblock \emph{Journal of Machine Learning Research}, 24\penalty0 (89):\penalty0 1--97, 2023.

\bibitem[Lavin et~al.(2021)Lavin, Krakauer, Zenil, Gottschlich, Mattson, Brehmer, Anandkumar, Choudry, Rocki, Baydin, et~al.]{lavin2021simulation}
Lavin, A., Krakauer, D., Zenil, H., Gottschlich, J., Mattson, T., Brehmer, J., Anandkumar, A., Choudry, S., Rocki, K., Baydin, A.~G., et~al.
\newblock Simulation intelligence: Towards a new generation of scientific methods.
\newblock \emph{arXiv preprint arXiv:2112.03235}, 2021.

\bibitem[Li et~al.(2020)Li, Kovachki, Azizzadenesheli, Liu, Bhattacharya, Stuart, and Anandkumar]{li2020neural}
Li, Z., Kovachki, N., Azizzadenesheli, K., Liu, B., Bhattacharya, K., Stuart, A., and Anandkumar, A.
\newblock Neural operator: Graph kernel network for partial differential equations.
\newblock \emph{arXiv preprint arXiv:2003.03485}, 2020.

\bibitem[Li et~al.(2021)Li, Kovachki, Azizzadenesheli, Liu, Bhattacharya, Stuart, and Anandkumar]{li2021fourier}
Li, Z., Kovachki, N.~B., Azizzadenesheli, K., Liu, B., Bhattacharya, K., Stuart, A.~M., and Anandkumar, A.
\newblock Fourier neural operator for parametric partial differential equations.
\newblock In \emph{International Conference on Learning Representations (ICLR)}, 2021.

\bibitem[Li et~al.(2023)Li, Shu, and Barati~Farimani]{scalabelTransformLi}
Li, Z., Shu, D., and Barati~Farimani, A.
\newblock Scalable transformer for pde surrogate modeling.
\newblock In Oh, A., Naumann, T., Globerson, A., Saenko, K., Hardt, M., and Levine, S. (eds.), \emph{Advances in Neural Information Processing Systems}, volume~36, pp.\  28010--28039. Curran Associates, Inc., 2023.

\bibitem[Liao et~al.(2023)Liao, Wood, Das, and Smidt]{liao2023equiformerv2}
Liao, Y.-L., Wood, B., Das, A., and Smidt, T.
\newblock Equiformerv2: Improved equivariant transformer for scaling to higher-degree representations.
\newblock \emph{arXiv preprint arXiv:2306.12059}, 2023.

\bibitem[Lino et~al.(2022)Lino, Fotiadis, Bharath, and Cantwell]{lino2022multi}
Lino, M., Fotiadis, S., Bharath, A.~A., and Cantwell, C.~D.
\newblock Multi-scale rotation-equivariant graph neural networks for unsteady eulerian fluid dynamics.
\newblock \emph{Physics of Fluids}, 34\penalty0 (8), 2022.

\bibitem[Palmer(2019)]{palmer2019stochastic}
Palmer, T.
\newblock Stochastic weather and climate models.
\newblock \emph{Nature Reviews Physics}, 1\penalty0 (7):\penalty0 463--471, 2019.

\bibitem[Passaro \& Zitnick(2023)Passaro and Zitnick]{passaro2023reducing}
Passaro, S. and Zitnick, C.~L.
\newblock Reducing so (3) convolutions to so (2) for efficient equivariant gnns.
\newblock In \emph{International Conference on Machine Learning}, pp.\  27420--27438. PMLR, 2023.

\bibitem[Paszke et~al.(2017)Paszke, Gross, Chintala, Chanan, Yang, DeVito, Lin, Desmaison, Antiga, and Lerer]{paszke2017automatic}
Paszke, A., Gross, S., Chintala, S., Chanan, G., Yang, E., DeVito, Z., Lin, Z., Desmaison, A., Antiga, L., and Lerer, A.
\newblock Automatic differentiation in pytorch.
\newblock 2017.

\bibitem[Peir{\'o} \& Sherwin(2005)Peir{\'o} and Sherwin]{peiro2005finite}
Peir{\'o}, J. and Sherwin, S.
\newblock Finite difference, finite element and finite volume methods for partial differential equations.
\newblock In \emph{Handbook of Materials Modeling: Methods}, pp.\  2415--2446. Springer, 2005.

\bibitem[Preparata \& Shamos(2012)Preparata and Shamos]{preparata2012computational}
Preparata, F.~P. and Shamos, M.~I.
\newblock \emph{Computational geometry: an introduction}.
\newblock Springer Science \& Business Media, 2012.

\bibitem[Raissi et~al.(2019)Raissi, Perdikaris, and Karniadakis]{raissi2019physics}
Raissi, M., Perdikaris, P., and Karniadakis, G.~E.
\newblock Physics-informed neural networks: A deep learning framework for solving forward and inverse problems involving nonlinear partial differential equations.
\newblock \emph{Journal of Computational physics}, 378:\penalty0 686--707, 2019.

\bibitem[Sanchez-Gonzalez et~al.(2020)Sanchez-Gonzalez, Godwin, Pfaff, Ying, Leskovec, and Battaglia]{sanchez2020learning}
Sanchez-Gonzalez, A., Godwin, J., Pfaff, T., Ying, R., Leskovec, J., and Battaglia, P.
\newblock Learning to simulate complex physics with graph networks.
\newblock In \emph{International conference on machine learning}, pp.\  8459--8468. PMLR, 2020.

\bibitem[Van~Gunsteren \& Berendsen(1990)Van~Gunsteren and Berendsen]{van1990computer}
Van~Gunsteren, W.~F. and Berendsen, H.~J.
\newblock Computer simulation of molecular dynamics: methodology, applications, and perspectives in chemistry.
\newblock \emph{Angewandte Chemie International Edition in English}, 29\penalty0 (9):\penalty0 992--1023, 1990.

\bibitem[Vaswani et~al.(2017)Vaswani, Shazeer, Parmar, Uszkoreit, Jones, Gomez, Kaiser, and Polosukhin]{vaswani2017attention}
Vaswani, A., Shazeer, N., Parmar, N., Uszkoreit, J., Jones, L., Gomez, A.~N., Kaiser, {\L}., and Polosukhin, I.
\newblock Attention is all you need.
\newblock \emph{Advances in neural information processing systems}, 30, 2017.

\bibitem[Weiler \& Cesa(2019)Weiler and Cesa]{weiler2019general}
Weiler, M. and Cesa, G.
\newblock General e (2)-equivariant steerable cnns.
\newblock \emph{Advances in neural information processing systems}, 32, 2019.

\bibitem[Weiler et~al.(2018)Weiler, Geiger, Welling, Boomsma, and Cohen]{weiler20183d}
Weiler, M., Geiger, M., Welling, M., Boomsma, W., and Cohen, T.~S.
\newblock 3d steerable cnns: Learning rotationally equivariant features in volumetric data.
\newblock \emph{Advances in Neural Information Processing Systems}, 31, 2018.

\bibitem[Worrall et~al.(2017)Worrall, Garbin, Turmukhambetov, and Brostow]{worrall2017harmonic}
Worrall, D.~E., Garbin, S.~J., Turmukhambetov, D., and Brostow, G.~J.
\newblock Harmonic networks: Deep translation and rotation equivariance.
\newblock In \emph{Proceedings of the IEEE conference on computer vision and pattern recognition}, pp.\  5028--5037, 2017.

\bibitem[Zhu \& Zabaras(2018)Zhu and Zabaras]{zhu2018bayesian}
Zhu, Y. and Zabaras, N.
\newblock Bayesian deep convolutional encoder--decoder networks for surrogate modeling and uncertainty quantification.
\newblock \emph{Journal of Computational Physics}, 366:\penalty0 415--447, 2018.

\end{thebibliography}
\bibliographystyle{icml2024}

\newpage
\appendix

\onecolumn

\section{Method: Additional details}\label{add:method}
In this section, we provide some additional information on the method section of the main paper. Figure \ref{fig:model_overviews2} shows an overview of the message-passing and feed-forward layers in Section \ref{sec:method}.
\subsection{An equivariance proof}\label{app:proof_equivariance}
We prove Proposition \ref{prop:equivariance}, which states that if if $\mR_{-\alpha_i}$ is a rotation to align $\vx_i, \vr_i$ to the $x$-axis, and $\mR_{\alpha_i}$ is its inverse, then for a function $f = \mR_{\alpha_i} \circ g \circ \mR_{-\alpha_i}$, where $g$ is any nonlinear function, it holds that
\begin{equation}
    \mR  f\left( \vx_i ,\vr_i\right) = f\left (\mR\vx_i,\mR\vr_i \right), \quad \forall \mR \in SO(2).
\end{equation}
Thus, $f$ is equivariant under the SO(2) group.
\begin{proof} 
A function $f$ is equivariant to the rotation if $\mR_{-\alpha_i}$ is a rotation to align $\vx_i, \vr_i$ to the $x$-axis, and $\mR_{\alpha_i}$ is its inverse. Then, for a function $f = \mR_{\alpha_i} \circ g \circ \mR_{-\alpha_i}$, where $g$ is any nonlinear function, it holds that
\begin{equation}
    \mR  f\left( \vx_i ,\vr_i\right) = f\left (\mR\vx_i,\mR\vr_i \right), \quad \forall \mR \in SO(2),
\end{equation}
\begin{equation}\label{eq:proof1}
    \mR_\beta f( \vx_i; r_i,  \gamma_i) =  f( \mR_\beta\vx_i; r_i,  \gamma_i + \beta), \quad \forall \mR_\beta \in SO(2),
\end{equation}
where $\beta$ is the angle that $\mR$ rotates with. We represent the node by its attribute $\vx_i$ and its position in polar coordinates $\vr_i = (r_i, \gamma_i)$, with $r_i$ its distance from the origin and $\gamma_i$ the angle to the $x$-axis. We express the position in polar coordinates to track where the node $i$ is located easily. If we rotate $\vx_i$ by $\beta$ degrees, the node position changes to $(r_i, \gamma_i + \beta)$.

We define the function $ f $ as a composition of three functions: a rotation $\mR_{-\alpha_{i}}$ to align with the $x$-axis, followed by a nonlinear function $  g $, and finally a rotation $\mR_{\alpha_{i}}$ back to the original position, i.e.
\begin{equation}\label{eq_proof3}
    f =\mR_{\alpha_{i}} \circ g \circ \mR_{-\alpha_{i}},
\end{equation}
where $\alpha_{i}$ is the angle to the $x$-axis that $\mR_{\alpha_{i}}$ acts on.

To prove that \eqq{eq:proof1} holds, we write down the equation and then verify that the left and right sides are equal: 
\begin{equation}\label{eq:proof2}
    \left(\mR_\beta \circ \mR_{\alpha_{i}} \circ g \circ \mR_{-\alpha_{i}}\right) \vx_i:r_i, \gamma_i =  \left(\mR_{\alpha_{i}} \circ g \circ \mR_{-\alpha_{i}} \circ \mR_\beta \right) \vx_i;r_i, \gamma_i.
\end{equation}

For the left-hand side of \eqq{eq:proof2}, we have:
\begin{align}
    \left(\mR_\beta \circ \mR_{\alpha_{i}} \circ g \circ \mR_{-\alpha_{i}}\right) \vx_i; r_i, \gamma_i &= 
    \left(\mR_\beta \circ \mR_{\alpha_{i}} \circ g \right)\mR_{-\alpha_{i}} \vx_i;r_i, \gamma_i-\alpha_i \quad (\text{but } \alpha_i = \gamma_i)\nonumber \\
 &=\left(\mR_\beta \circ \mR_{\gamma_{i}}\right)  g \left(\mR_{-\gamma_{i}} \vx_i;r_i, 0 \right) \nonumber\\
 &=\left(\mR_{\beta +  \gamma_i} \right)  g \left(\mR_{-\gamma_{i}} \vx_i;r_i, 0 \right). \nonumber
\end{align}
For the right-hand side of \eqq{eq:proof2}, we have:
\begin{align}
    \left(\mR_{\alpha_{i}} \circ g \circ \mR_{-\alpha_{x}} \circ \mR_\beta \right) \vx_i; r_i, \alpha_i &= \left(\mR_{\alpha_{i}} \circ g \circ \mR_{-\alpha_{i}} \right)   \mR_\beta\vx_i; r_i, \gamma_i + \beta \nonumber\\ 
    &= \left(\mR_{\alpha_{i}} \circ g \right) \mR_{-\alpha_{i}}    \mR_\beta\vx_i; r_i, \gamma_i + \beta -\alpha_{i} \nonumber  \quad (\text{but } \alpha_i = \gamma_i + \beta)\\
    &= \left(\mR_{\gamma_{i} + \beta} \circ g \right) \mR_{-\gamma_i}\vx_i; r_i, 0 \nonumber \\
    &=\left(\mR_{\gamma_i + \beta}  \right)g(\mR_{-\alpha_i}\vx_i; r_i, 0). \nonumber
\end{align}
Note that $\alpha_{i} = \gamma_i + \beta$, since $\mR_{\alpha_{i}}$ acts on $\mR_\beta\vx_i$.

Since the left and right sides of \eqq{eq:proof2} are equal for all $\beta \in (0, 2\pi)$, this concludes the proof that the function $ f $ is equivariant under rotations.
\end{proof}

\begin{figure}[t]
    \centering
    \begin{subfigure}
        \centering
        \includegraphics[width = 0.25\linewidth]{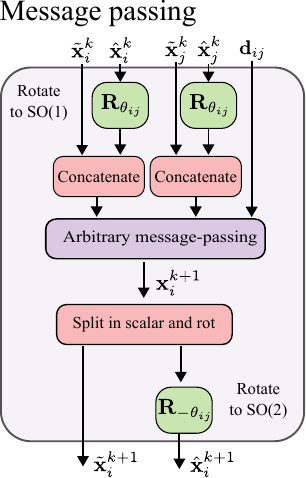}
    \end{subfigure}%
    \hspace{1in}
    \begin{subfigure}
        \centering
        \includegraphics[width = 0.25\linewidth]{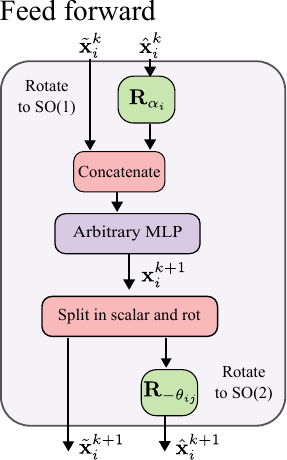}
    \end{subfigure}
    \caption{On overview of the message-passing layer (\textbf{left}) and the feed-forward layer (\textbf{right}).}
    \label{fig:model_overviews2}
\end{figure}

\subsection{SE(2) activation function}

We do not only have to use the relationship in Proposal \ref{prop:equivariance} to create an MLP; we can also use it to construct an SE(2) activation function as:
\begin{align}\label{eq:activation}
    \hat{\vx}_i \oplus \hat{\vx}_i^* &= \text{Activation} \left ( \tilde{\vx}_i \oplus \boldsymbol{R}_{\alpha_i} \hat{\vx}_i\right), \\
    \vx_i &=\tilde{\vx}_i \oplus \boldsymbol{R}_{-\alpha_i} \hat{\vx}_i^*. \nonumber
\end{align}
We can use any activation function, and the equivariance criteria still hold. 

\subsection{Linear layer}
In the same way as for the activation function above, we can create a SO(2) equivariant linear layer:
\begin{align}
    \tilde{\vx}_i^\text{out} \oplus \hat{\vx}_i^{r,\text{out}} &= \text{Linear} \left ( \tilde{\vx}_i^\text{in} \oplus \boldsymbol{R}_{\alpha_i} \hat{\vx}_i^\text{in}\right), \\
     \hat{\vx}_i^\text{out} \oplus \hat{\vx}_i^\text{out} &=\tilde{\vx}_i \oplus \boldsymbol{R}_{-\alpha_i} \hat{\vx}_i^{r,\text{out}}.
\end{align}

Figure \ref{fig:linear_trans} illustrates how this procedure remains equivariant to rotation in the 2D plane. This is because the linear transformation is performed after the features have been aligned to the $x$-axis, making it invariant to the shape's original rotation.
 
\begin{figure}[h]
    \centering
    \begin{subfigure}
        \centering
        \includegraphics[height=1.25in]{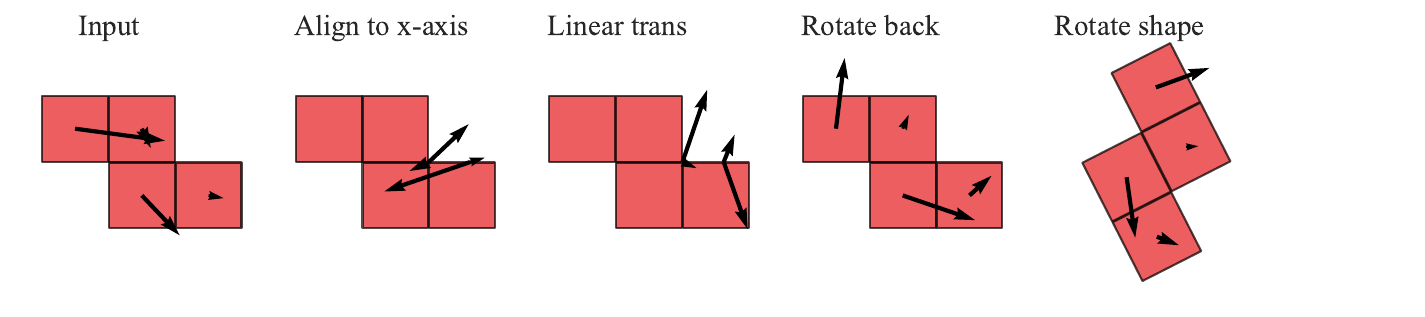}
    \end{subfigure}%
    \begin{subfigure}
        \centering
        \includegraphics[height=1.25in]{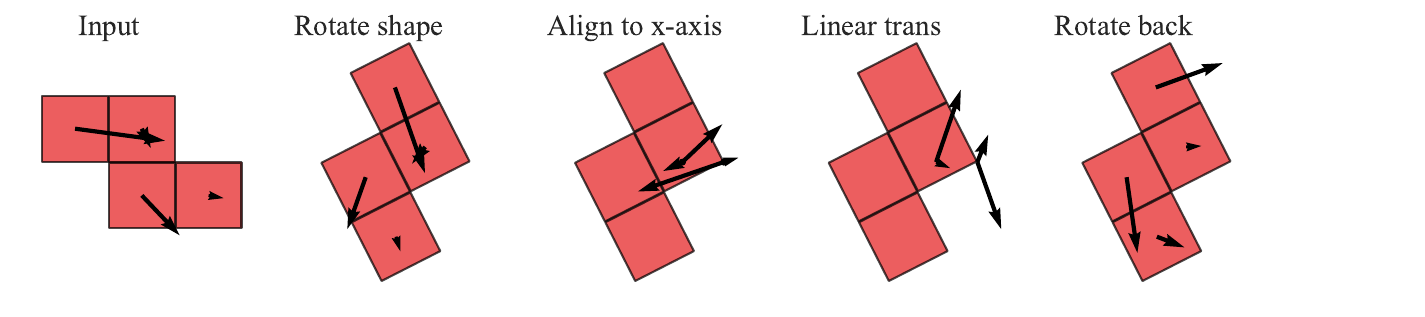}
    \end{subfigure}
    \caption{Illustration of how the linear layer acts on the rotation input features. The input is rotated \emph{before} the linear layer in the \textbf{top} row, while in the \textbf{button} row, the shape is rotated \emph{after} the linear layer.  
    Since the features always align with the $x$-axis before the linear layers, the input to the linear layer is independent of the rotation of the feature, i.e., it is equivariant.}
    \label{fig:linear_trans}
\end{figure}

\section{Additional experiments}
In this section we conduct two additional experiments. The first experiment compares the equivariance error of our version of an activation function in \eqq{eq:activation} to using point-wise sampled nonlinearity \citep{dehaan2021, passaro2023reducing,liao2023equiformerv2}. The second experiment compares our projected SO(2) procedure to achieve equivariance to one that does not use projection on the Navier-Stokes data. We also give implementation details on the Navier-Stokes experiments and some additional results for the experiments in the main paper.

\begin{figure}[h]
\vskip 0.2in
\begin{center}
\centerline{\includegraphics[width=0.5\columnwidth]{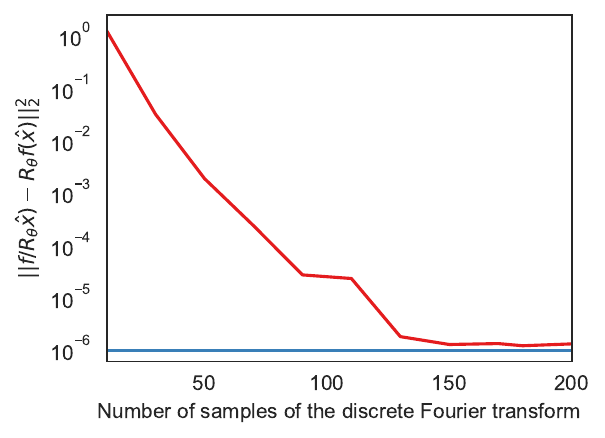}}
\caption{The equivariance error of the model using the pointwise Fourier nonlinearity with different numbers of samples and the rotational nonlinearity proposed in this paper. We used a two-layer message-passing model for this test. The curve is the mean from 50 forward passes of the model.}
\label{fig:eq_error}
\end{center}
\vskip -0.2in
\end{figure}

\subsection{Comparing rotation activation to point-wise sampled non-linearity}
A common way of creating an equivariant activation function is using point-wise sampled nonlinearities \citep{dehaan2021, passaro2023reducing,liao2023equiformerv2}. The activation first converts vectors of all degrees to point samples on
a sphere, applies unconstrained functions $f$ (such as an activation function) to those samples, and finally converts
them back to vectors. Specifically, given an input of rotation features $\hat{\vx}_i$, then, 
\begin{equation}\label{eq:point_wise}
    \hat{\vx}_i = G^{-1}  \left (f(G(\hat{\vx}_i))\right )
\end{equation} 
where $G$ denotes the conversion from vectors to point samples on a sphere. The equivariance error of this activation function depends on the number $n_s$ of sampled points on the sphere; the equivariance error decreases when we increase the number of samples.

The equivariance error of our proposed rotation activation function \eqq{eq:activation} only depends on the numerical accuracy when deriving the rotation matrix $\boldsymbol{R}_{\theta_i}$. Figure \ref{fig:eq_error} compares the equivariance error between our approach and a Fourier sample-based method (with Leaky ReLU as activation function) proposed by \citet{dehaan2021}. We see that the sample-based method reaches the equivariance error of our SE(2)-activation when the number of samples increases, but at a memory cost.

\subsection{Additional results from the Tetris experiment}\label{app:tetris}
Figure \ref{fig:tetris_test_blocks} shows some rotated shapes in the Tetris test dataset. In Figure \ref{fig:tetris_test_error}, the model's test error at different epochs is plotted, comparing the different training datasets. We can see that \emph{SE2Conv-MLP} quickly converges to a perfect test accuracy of 1.

\begin{figure}[ht]
\vskip 0.2in
\begin{center}
\centerline{\includegraphics[width=0.6\columnwidth]{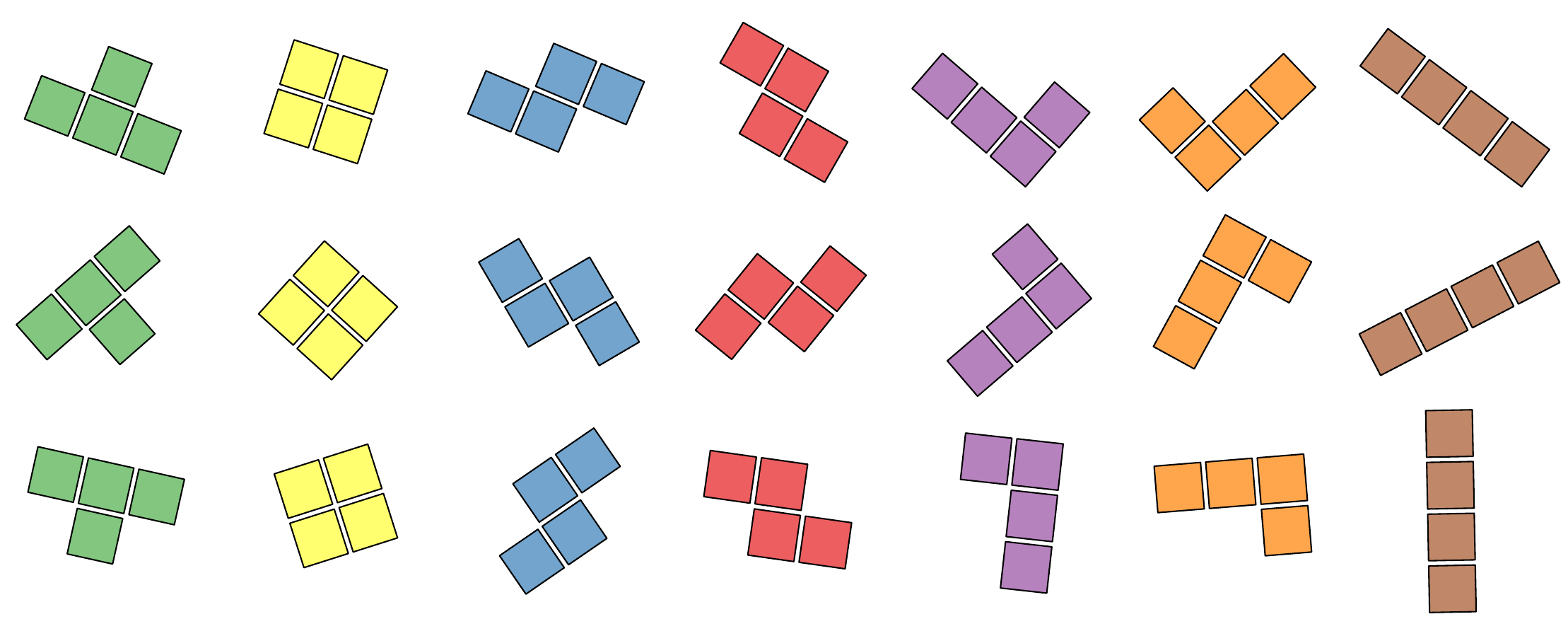}}
\caption{Examples from the test dataset in the Tetris classification experiment.}
\label{fig:tetris_test_blocks}
\end{center}
\vskip -0.2in
\end{figure}

\begin{figure}[ht]
\vskip 0.2in
\begin{center}
\centerline{\includegraphics[width=0.7\columnwidth]{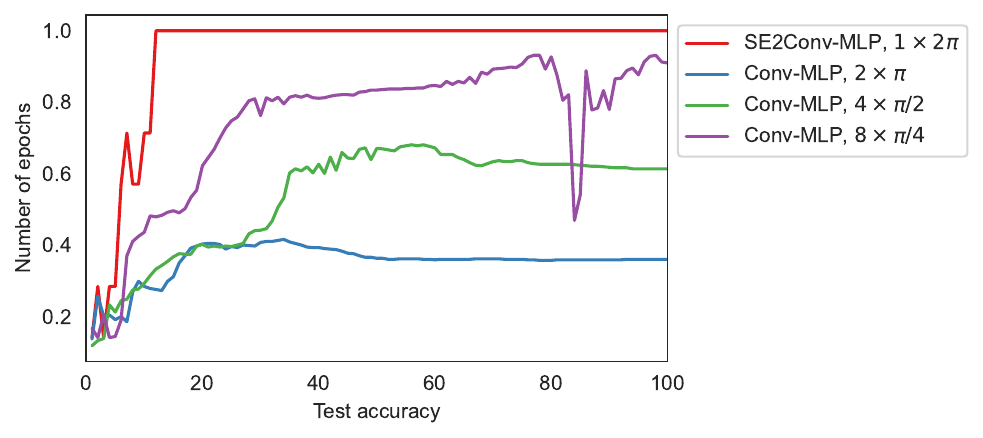}}
\caption{The test accuracy per epoch for the Tetris experiment.}
\label{fig:tetris_test_error}
\end{center}
\vskip -0.2in
\end{figure}

\subsection{Navier-Stokes simulations}\label{app:navier_extra}

\subsubsection{Implementation and training details}
All models are implemented using Pytorch \citep{paszke2017automatic} and Pytorch Geometric \citep{FeyLenssen2019}, and the experiments were conducted using four NVIDIA A100 GPUs. The models were trained using the Adam optimizer \citep{KingBa15}. We used a batch size of 32 (8 per GPU node) and trained for 500 epochs. We started with a learning rate of $1 \times 10^{-3}$, which was decreased during training using cosine annealing. A validation set consisting of 5\% of the training dataset is used for this purpose. Our final model uses the model with the lowest validation error.

We present additional details of our models used in all the Navier-Stokes experiments in Table \ref{tab:model_complexity}. We can see that the time it takes to do a forward pass for the SE(2) models is two to three times longer than for their invariant counterparts. This additional time overhead is mainly due to the derivation of the rotation matrix. Although, for being an equivariant model, the timing of our SE2-Conv-models is low, as we will illustrate in Section \ref{app:additional_navier_expermint}.
\begin{table}[h]
\caption{Data of the respective model describing their model size, simulation time, and memory consumption. The scalar hidden dimension of the invariant models has been scaled up compared to the equivariant model to match their model size.}
\label{tab:model_complexity}
\vskip 0.15in
\begin{center}
\begin{footnotesize}
\begin{sc}
\begin{tabular}{lcccc}
\toprule
  &  SE2Conv-Trans &  Conv-Trans&  SE2Conv-MLP& Conv-MLP\\
\midrule
Nbr model parameters    &6.0e6 & 7.6e6&6.0e6 &7.6e6\\
Time (seconds), forward pass &0.020 & 0.013&0.018&0.006\\
Num layers &7 & 7&7&7\\
Hidden dim scalar &64 & 256&64 &256\\
Hidden dim vector (Nbr irreps with $m=1$) &64 & 0&64&0\\
\bottomrule
\end{tabular}
\end{sc}
\end{footnotesize}
\end{center}
\vskip -0.1in
\end{table}

\subsubsection{Validation metrics}
The one-step error is defined as:
\begin{equation}\label{eq:smse_loss3}
    \mathcal{L}_{SMSE} = \frac{1}{N_b}\frac{1}{N_t-3}\sum_{k=4}^{N_t}\sum_{i=1}^{N_b} (u_i^k- u_i^{t,k})^2 + (v_{i,x}^k-v_{i,x}^{t,k})^2 + (v_{i,y}^k-v_{i,y}^{t,k})^2.
\end{equation}

Here, $N_b$ is the number of nodes in the graph/mesh, and $N_t$ is the number of timesteps in the trajectory. We use three timesteps as an input to our model, which means we do not get a prediction for the three first steps in our trajectory. 

The rollout error is calculated similarly, but with a key difference: the "true" values are only used to predict the first timestep at $ k = 4 $. After that, we use the predicted outputs as inputs. For example, to predict $u_i^5 $, the model takes as input the true fields at $ u_i^{2,t} $ and $ u_i^{3,t} $, along with the predicted output from the previous timestep, $ u_i^4 $. This process continues for subsequent timesteps.

\subsubsection{Additional Navier-Stokes experiments}\label{app:additional_navier_expermint}

Beyond the results presented in the main article, we conducted an experiment to investigate the expressivity of our specific method in the feed-forward and message-passing layer for achieving SO(2) equivariance. Our method involves projecting the node to a principal axis and concatenating the rotational and scalar features before the MLP. To assess this, we compared our approach to an SO(2) MLP that does not project the rotational features and cannot easily combine rotational and scalar features, namely, the SO(2) MLP, proposed in escnn \url{https://github.com/QUVA-Lab/escnn} \citep{cesa2022a}, which we call MLP-escnn. 

The idea for their MLP is that if the rotational input $\hat{\vx} =  \hat{\vx}_i^1\oplus \hdots \oplus  \hat{\vx}^{N_r}$ consist of $N_r$ representations (2-dim vectors), then, we can linearly  combine these, using learned rotational matrices such that
\begin{equation}
    \hat{\vx}^i_\text{out} =  \sum_{j=1}^{N_r}\mR_{\delta_{ij}}\hat{\vx}^j,
\end{equation}
where
\begin{equation}
    \mR_{\delta_{ij}} = \begin{bmatrix}
        w_{ij}^1 &-w_{ij}^2 \\
        w_{ij}^2& w_{ij}^1
    \end{bmatrix}.
\end{equation}
and $ w_{ij}^1,  w_{ij}^2$  are the learnable weights. This linear transformation is equivariant since it only consists of rotations. We can create many arbitrary output representations $N_r^\text{out}$, such that  $\hat{\vx}_\text{out} = \hat{\vx}^1_\text{out}\oplus \hdots \oplus \hat{\vx}^{N_r^\text{out}}_\text{out} $. Using the point-wise sampled nonlinearity in \eqq{eq:point_wise}, we avoid projections. For the scalar part, we use a regular MLP.

Compared to ours, a downside of this approach is that we never mix the rotational and scalar features. To overcome this, we let the parameters of the rotation matrix in the message passing layer depend on the scalar features, such that $[ w_{ij}^1,w_{ij}^2] = \text{MLP}(\tilde{\vx})$.

We construct an experiment where the projected MLP with these MLP-escnn modules in the feed-forward blocks and the message passing layers, and compare these to an \mlp. All models have three layers, 32 scalar features, and 16 rotational features. All MLPs, also the rotational ones, have a hidden dimension three times the output dimension. We train the models on the Navier-Stokes dataset with varying force and 512 samples. 5\% of these samples are selected as validation data.

\begin{table}[h]
\caption{The time it takes to train one epoch with the different models.}
\label{tab:times_diff}
\vskip 0.15in
\centering
\begin{tabular}{lc}
\toprule
                model &   Time (seconds) \\
\midrule
Our \mlp. & 1.93 \\
Our \mlp. with MLP-escnn convolution& 19.3 \\
Our \mlp. with MLP-escnn feed-forward& 1.92 \\
MLP-escnn convolution and MLP-escnn feed-forward& 19.7 \\
\bottomrule
\end{tabular}
\end{table}
\begin{figure}[h]
\vskip 0.2in
\begin{center}
\centerline{\includegraphics[width=0.9\columnwidth]{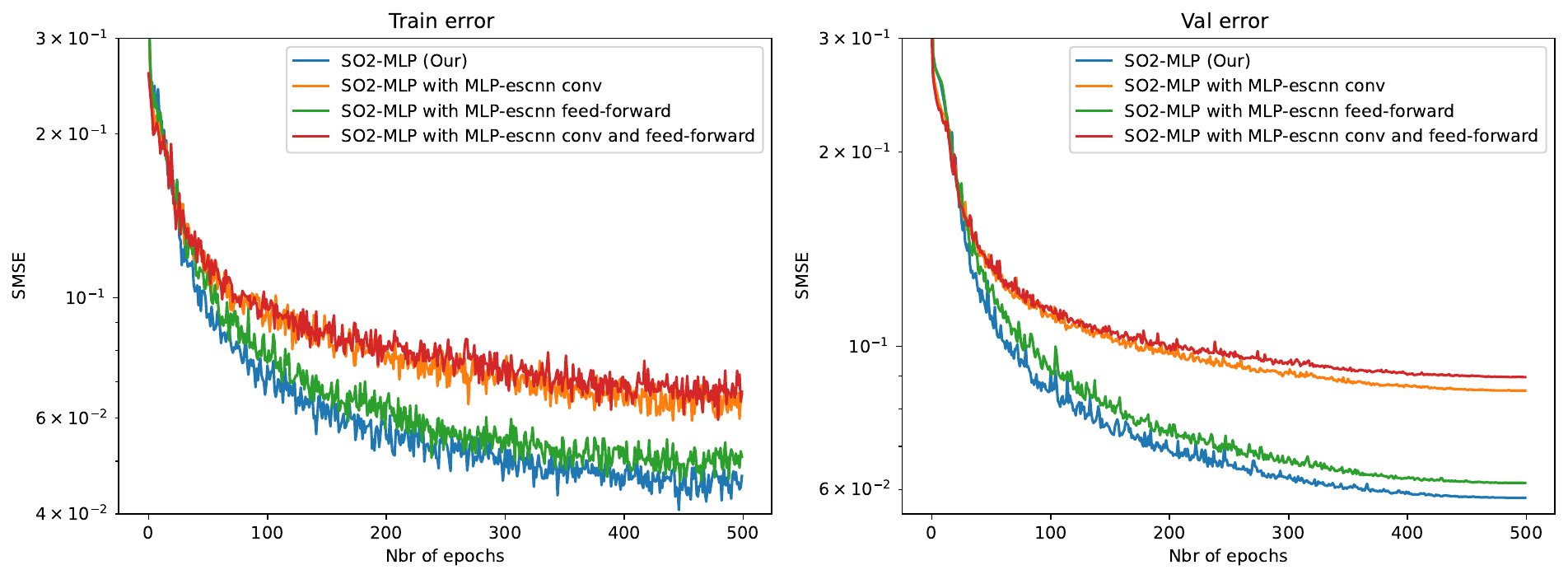}}
\caption{We conducted experiments comparing the train and validation errors of our projected SO(2) MLP to a non-projected SO(2) MLP, specifically MLP-escnn. We replaced the MLPs in both the feed-forward and message-passing layers in these experiments. We compared the training and validation errors to understand the models' expressivity.}
\label{fig:val_diff}
\end{center}
\vskip -0.2in
\end{figure}

Table \ref{tab:times_diff} shows the duration of each epoch for each model type. Notably, the training times for MLP-escnn-based message-passing layers are significantly longer. Figure \ref{fig:val_diff} plots the training and validation errors. Our projected method fits the training dataset better and generalizes the validation dataset more effectively. Implementation details can be found at \url{https://github.com/mariabankestad/SE2-GNN}.
\newpage
\subsubsection{Rollout plots}
We include additional rollout results for both constant buoyancy force (see Figure \ref{fig:roll_out_long}) and varying buoyancy force (see Figures \ref{fig:sim_u} and \ref{fig:sim_v}). The predictions are made using the SE2Conv-Trans model, which was trained on a dataset of 2048 samples and compared to its non-equivariant counterpart, Conv-Trans. Unlike in the training and validation datasets, the nodes are not sampled randomly in these plots. Instead, we have (from trajectories in the test dataset) sampled 1024 nodes using furthest point sampling (FSP) to get more equally spaced nodes for a better visual appearance.

\begin{figure}[h]
\vskip 0.2in
\begin{center}
\centerline{\includegraphics[width=0.5\columnwidth]{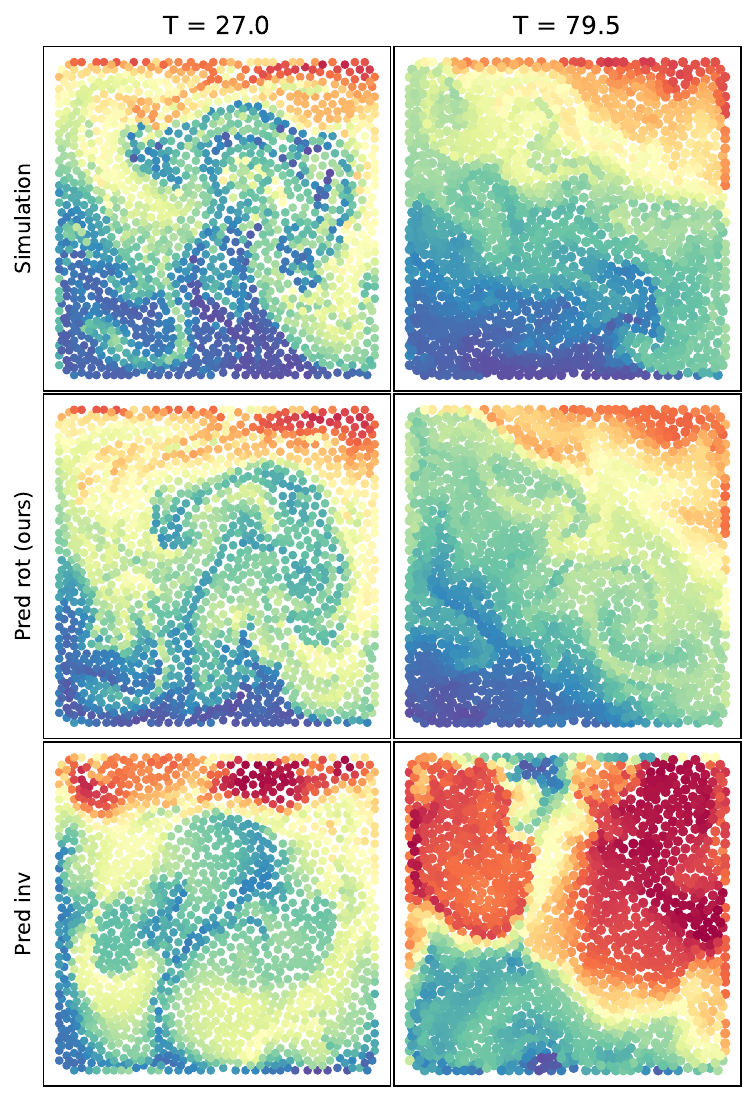}}
\caption{Top: Results using time-consuming simulation (here considered ground truth). Middle: Results for the corresponding timesteps for our approach. Bottom: Corresponding results for the baseline approach. Note how our approach maintains many of the main structures seen in the ground truth, even at longer time horizons, compared to the baseline. For a live visualization of this figure, see \url{https://github.com/mariabankestad/SE2-GNN}.}
\label{fig:roll_out_long}
\end{center}
\vskip -0.2in
\end{figure}

\begin{figure}[ht]
    \centering
    \begin{subfigure}
        \centering
        \includegraphics[height=1.5in]{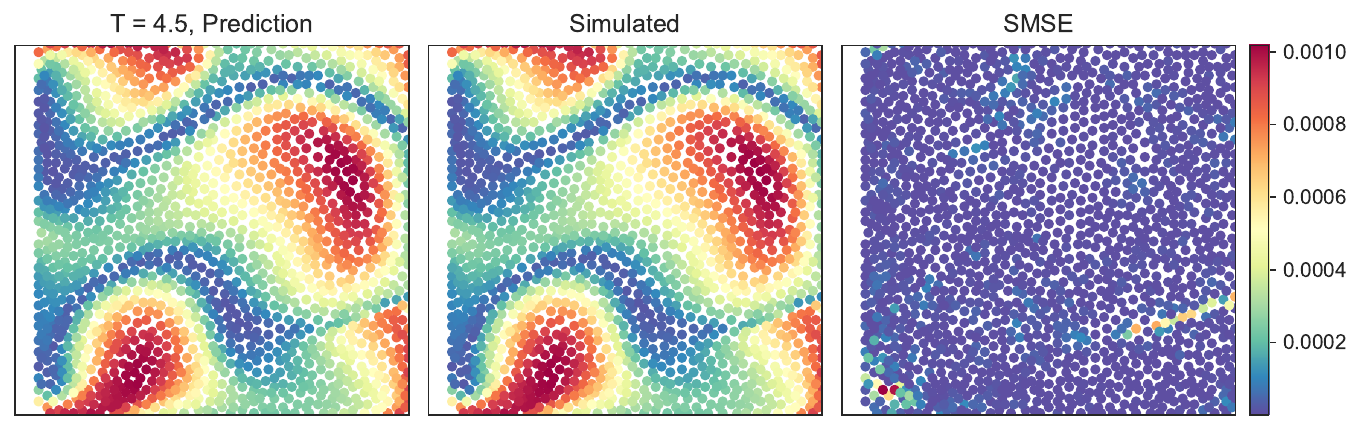}
    \end{subfigure}%
    \begin{subfigure}
        \centering
        \includegraphics[height=1.5in]{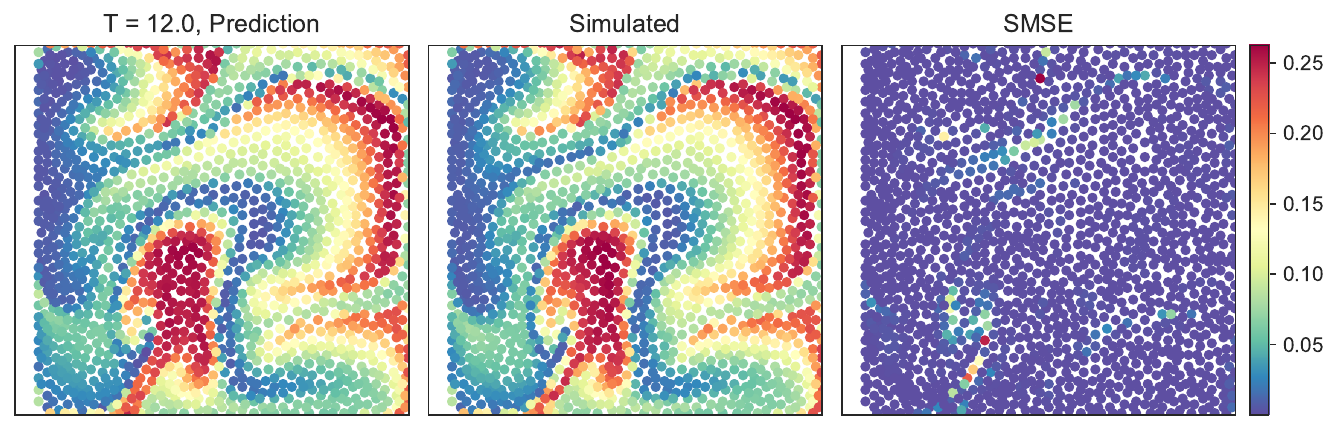}
    \end{subfigure}
    \begin{subfigure}
        \centering
        \includegraphics[height=1.5in]{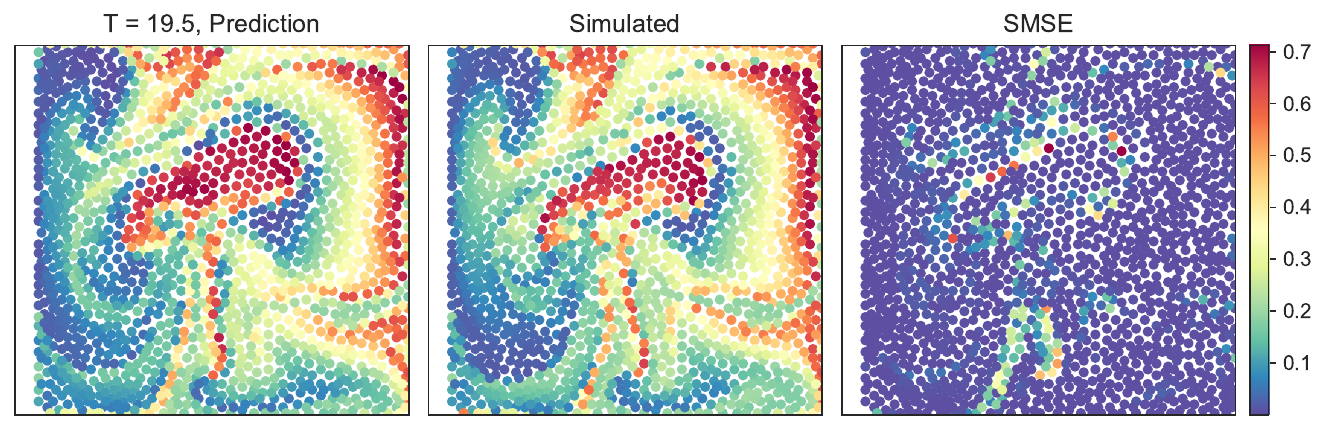}
    \end{subfigure}
    \caption{Rollout scalar field (smoke) results from the Navier-Stokes surrogate models with varying buoyancy forces in the simulated data.}
        \label{fig:sim_u}
\end{figure}

\begin{figure}
    \centering
    \begin{subfigure}
        \centering
        \includegraphics[height=1.5in]{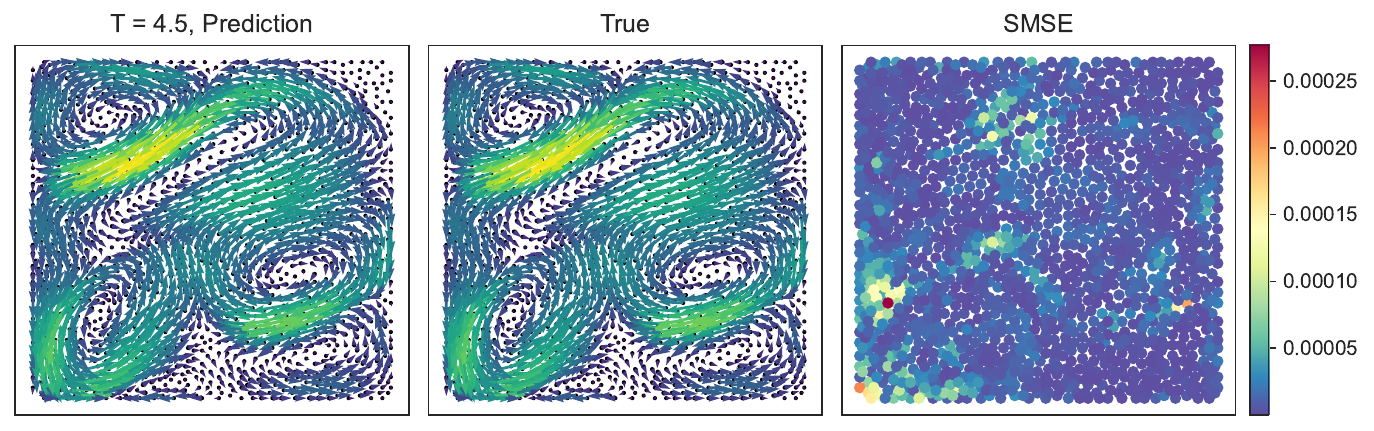}
    \end{subfigure}%
    \begin{subfigure}
        \centering
        \includegraphics[height=1.5in]{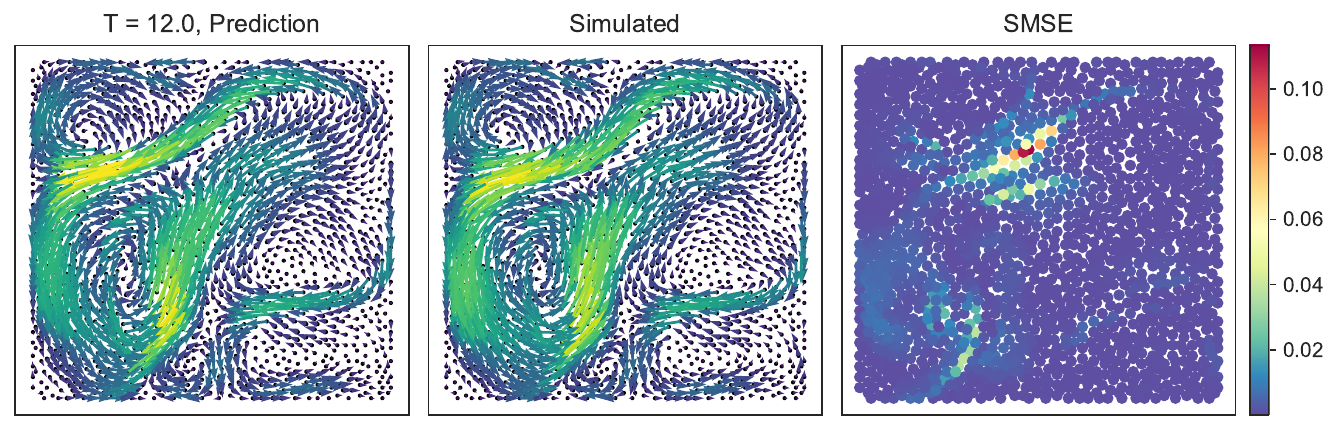}
    \end{subfigure}
    \begin{subfigure}
        \centering
        \includegraphics[height=1.5in]{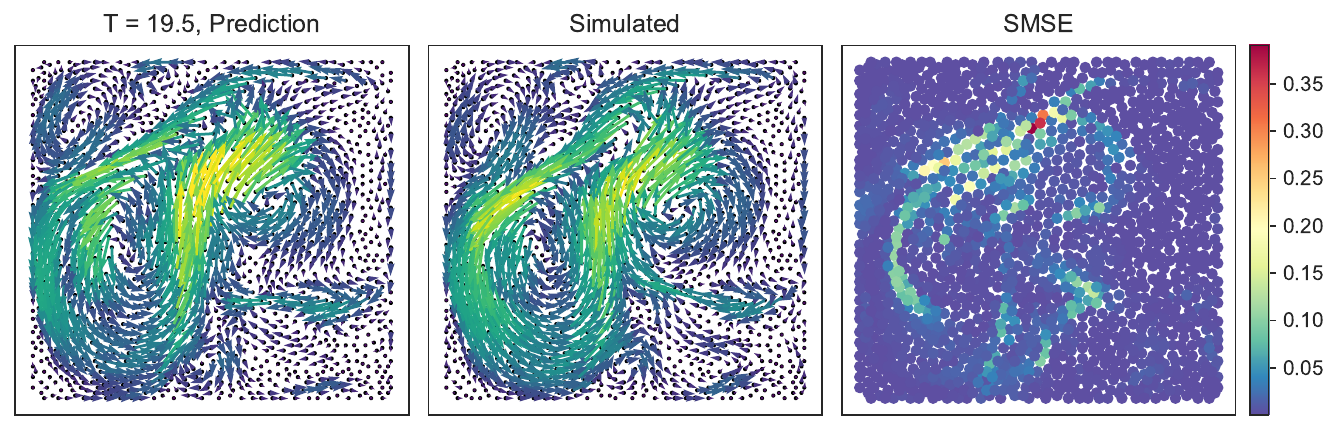}
    \end{subfigure}
    \caption{Rollout vector field results from the Navier-Stokes surrogate models with varying buoyancy forces in the simulated data. }
    \label{fig:sim_v}
\end{figure}

\section{ Simulation details}\label{add:navier_sim}
The Navier-Stokes simulation was created using $\Phi_{\text{FLOW}}$ \citep{holl2020phiflow}, a Python open-source simulation toolkit built for optimization and machine learning applications.

\textbf{Simulations with an obstacle.}
In the smoke simulation around an obstacle, we used a grid with a size of $200\times 200$ grid points, with $\Delta x, \Delta y = 0.5$ (see Figure \ref{fig:grid_obs}). We use a buoyancy force $\vf = (0,0.5)$, and the obstacle has the shape of a circle with a radius of 15 (the domain is $100 \times 100$). The y-central coordinate of the obstacle $y_\text{obs} = 50$, while the x-coordinate is discreetly uniformly sampled as $x_\text{obs} = \mathcal{U}\{20, 80\}$. We also add a smoke inlet with an intensity of 0.5 and a radius 7. The inlet is located y-location $y_\text{inlet} = 9.5$, while the x-location is discreetly uniformly sampled as $x_\text{inlet} =\mathcal{U}\{10, 90\}$. We simulate the smoke for 150 timesets, using MacCormack advection\citep{holl2020phiflow}, where $\Delta t = 0.5$ seconds. After that, we downsample the time by one using every other timestep, resulting in a time trajectory of 75 timesteps with $\Delta t = 1$. We create 512 of these trajectories for training and 16 for testing.

We randomly sampled 1024 nodes per grid to create our training data and used Delaunay triangulation to establish the node connections. Figure \ref{fig:irregular_grid} shows an example grid in the training data, with the smoke inlet also marked.
\begin{figure}[h]
\vskip 0.2in
\begin{center}
\centerline{\includegraphics[width=\textwidth]{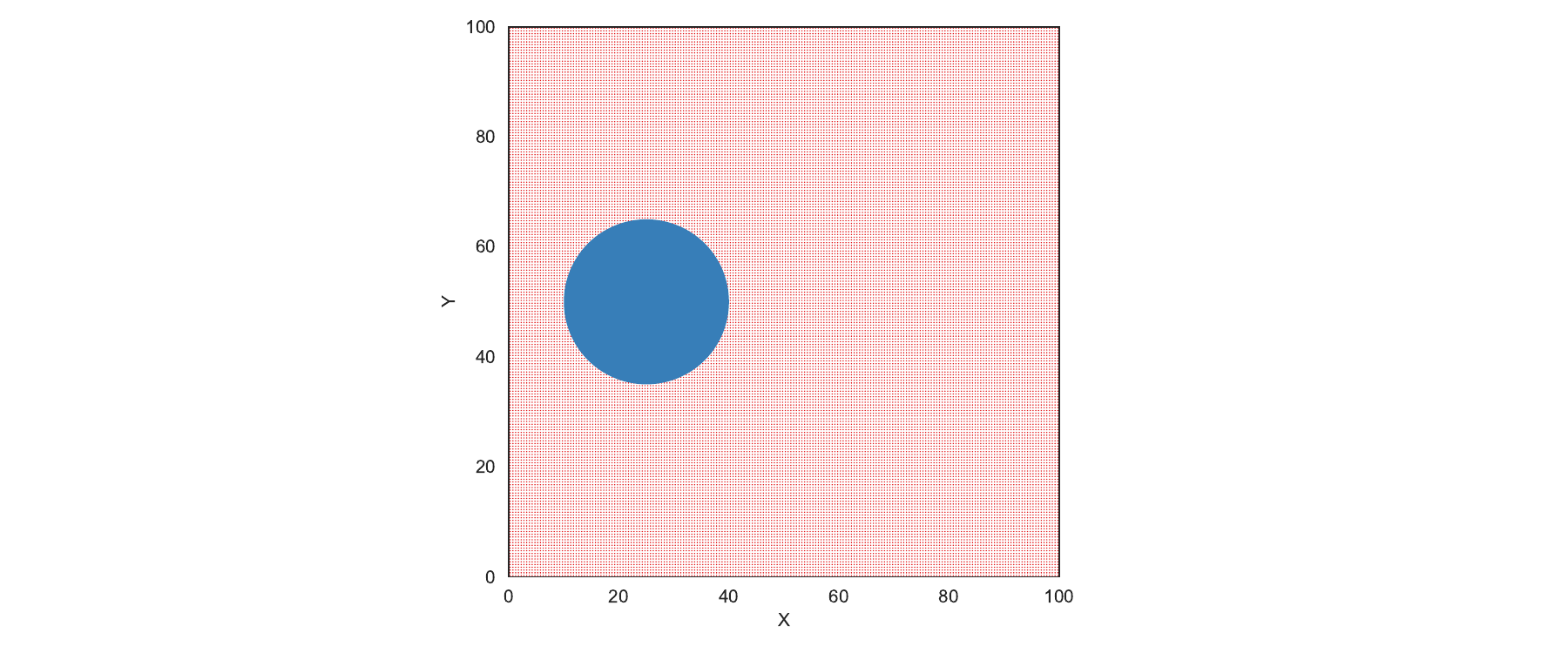}}
\caption{The regular grid with an obstacle that we use to simulate the Navier-Stokes equation.}
\label{fig:grid_obs}
\end{center}
\vskip -0.2in
\end{figure}

\begin{figure}[h]
\vskip 0.2in
\begin{center}
\centerline{\includegraphics[width=0.4\textwidth]{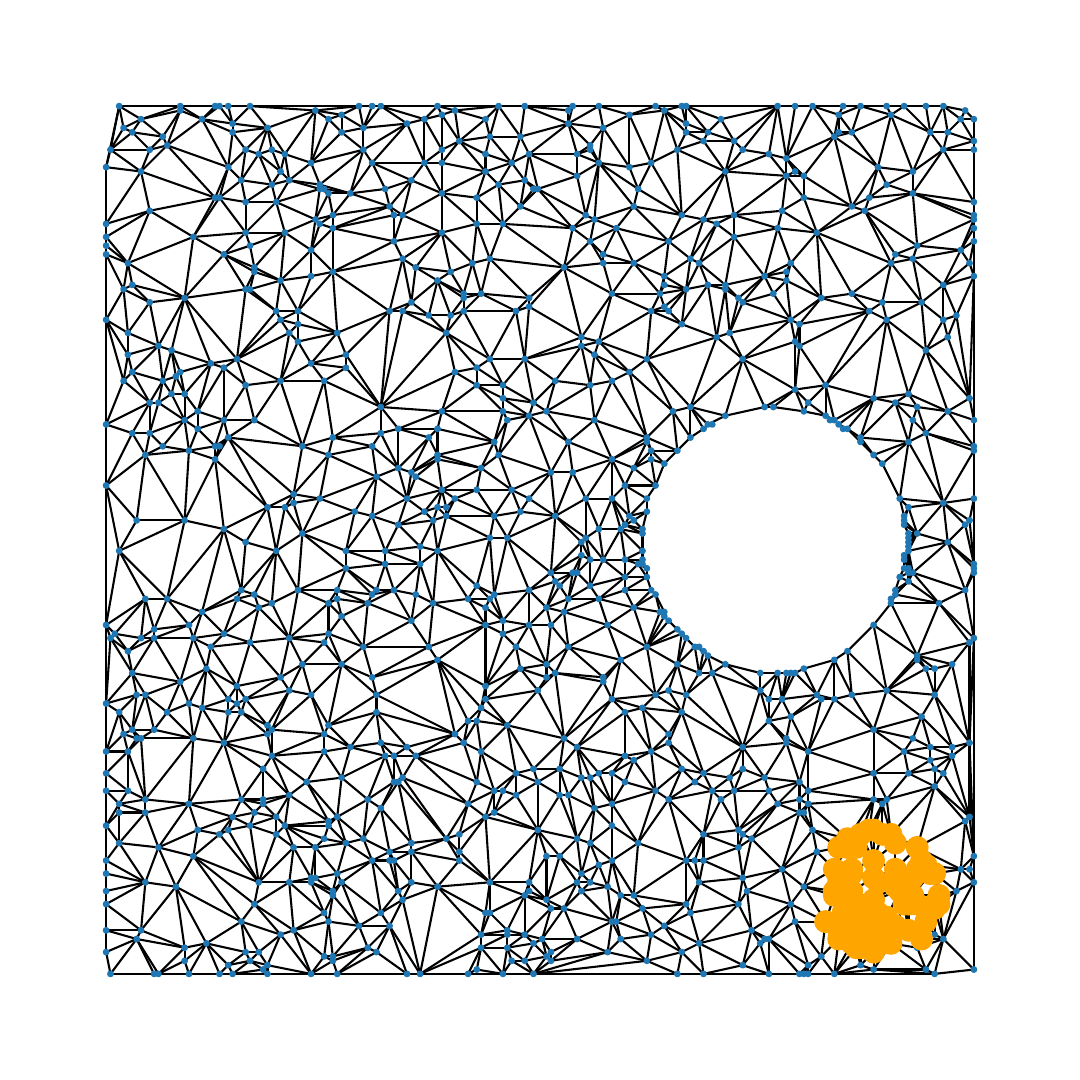}}
\caption{Example of the irregular grid with an obstacle, used for the machine learning model. The nodes where smoke is coming in are also marked.}
\label{fig:irregular_grid}
\end{center}
\vskip -0.2in
\end{figure}
\end{document}